\documentclass{article}

\usepackage{microtype}
\usepackage{graphicx}
\usepackage{booktabs}
\usepackage{array}
\usepackage{xspace}
\usepackage{hyperref}

\usepackage[preprint]{icml2027}

\usepackage{amsmath,amssymb,mathtools,amsthm}
\usepackage{enumitem}
\usepackage{tikz}
\usetikzlibrary{arrows.meta,positioning,calc,backgrounds,fit}
\usepackage[capitalize,noabbrev,nameinlink]{cleveref}

\newcommand{\Z}{\mathbb{Z}}

\newcommand{\calD}{\mathcal{D}}
\newcommand{\calP}{\mathcal{P}}
\newcommand{\tw}{\operatorname{tw}}
\newcommand{\err}{\operatorname{err}}
\newcommand{\relu}{\operatorname{ReLU}}
\newcommand{\Nin}{N^{-}}
\newcommand{\ind}{\mathbf{1}}
\newcommand{\tail}{\operatorname{tail}}
\newcommand{\head}{\operatorname{head}}
\newcommand{\onset}{\tau}
\newcommand{\reach}{\rightsquigarrow}
\newcommand{\FPT}{\ensuremath{\mathrm{FPT}}\xspace}
\newcommand{\Wone}{\ensuremath{\mathrm{W[1]}}\xspace}

\newcommand{\XP}{\ensuremath{\mathrm{XP}}\xspace}
\newcommand{\XNLP}{\ensuremath{\mathrm{XNLP}}\xspace}
\newcommand{\NP}{\ensuremath{\mathrm{NP}}\xspace}
\newcommand{\Ptime}{\ensuremath{\mathrm{P}}\xspace}
\newcommand{\paraNP}{\ensuremath{\mathrm{para\text{-}NP}}\xspace}
\newcommand{\DAGEDP}{\textnormal{\textsc{DAG-EDP}}\xspace}
\newcommand{\twoQNNT}{\textnormal{2-\textsc{QNNT}}\xspace}
\newcommand{\bfQNNT}{\textnormal{bias-free 2-\textsc{QNNT}}\xspace}
\newcommand{\BfQNNT}{\textnormal{Bias-free 2-\textsc{QNNT}}\xspace}
\newcommand{\dQNNT}{\textnormal{$d$-\textsc{QNNT}}\xspace}
\newcommand{\fA}{\mathsf{A}}
\newcommand{\fB}{\mathsf{B}}
\newcommand{\fC}{\mathsf{C}}
\newcommand{\fD}{\mathsf{D}}

\usepackage{aliascnt}
\theoremstyle{plain}
\newtheorem{theorem}{Theorem}[section]
\newcommand{\sharedthm}[3]{%
  \newaliascnt{#1}{theorem}%
  \theoremstyle{#3}\newtheorem{#1}[#1]{#2}%
  \aliascntresetthe{#1}%
  \crefname{#1}{#2}{#2s}}
\sharedthm{lemma}{Lemma}{plain}
\sharedthm{proposition}{Proposition}{plain}
\sharedthm{corollary}{Corollary}{plain}
\sharedthm{claim}{Claim}{plain}
\sharedthm{fact}{Fact}{plain}
\sharedthm{observation}{Observation}{plain}
\sharedthm{definition}{Definition}{definition}
\sharedthm{example}{Example}{definition}
\sharedthm{remark}{Remark}{remark}

\newcommand{\problembox}[3]{%
  \par\smallskip\noindent
  \parbox{\columnwidth}{%
    \textsc{#1}\\[2pt]
    \textit{Instance:} #2\\[1pt]
    \textit{Question:} #3}%
  \par\smallskip}

\hypersetup{
  pdftitle={Hardness and Tractability of Binary Quantized Neural Network Training in Low Dimensions},
  pdfauthor={Tao Jiang, Minbo Gao, Shaowei Cai},
  pdfsubject={Preprint},
  pdfkeywords={quantized neural networks, training complexity, parameterized
    complexity, W[1]-hardness, XNLP, disjoint paths, reachability},
}

\makeatletter
\renewcommand{\Notice@String}{\textit{Preprint.}}
\renewcommand{\printAffiliationsAndNotice}[1]{\global\icml@noticeprintedtrue%
  \stepcounter{@affiliationcounter}%
  {\let\thefootnote\relax\footnotetext{\hspace*{-\footnotesep}\ificmlshowauthors #1\fi%
      \forloop{@affilnum}{1}{\value{@affilnum} < \value{@affiliationcounter}}{
        \textsuperscript{\arabic{@affilnum}}\ifcsname @affilname\the@affilnum\endcsname%
          \csname @affilname\the@affilnum\endcsname%
        \else
          {\bf AUTHORERR: Missing \textbackslash{}icmlaffiliation.}
        \fi
      }.%
      { }E-mail: \texttt{\{jiangt,gaomb,caisw\}@ios.ac.cn}.

      \ \\
      \Notice@String
    }
  }
}
\makeatother

\icmltitlerunning{Hardness and Tractability of Binary Quantized Neural Network Training in Low Dimensions}

\begin{document}

\twocolumn[
  \icmltitle{Hardness and Tractability of Binary Quantized\texorpdfstring{\\}{}
    Neural Network Training in Low Dimensions}

  \icmlsetsymbol{equal}{*}

  \begin{icmlauthorlist}
    \icmlauthor{Tao Jiang}{equal,ios}
    \icmlauthor{Minbo Gao}{equal,ios}
    \icmlauthor{Shaowei Cai}{ios}
  \end{icmlauthorlist}

  \icmlaffiliation{ios}{Key Laboratory of System Software (Chinese Academy of
    Sciences); State Key Laboratory of Computer Science; Institute of Software,
    Chinese Academy of Sciences; School of Computer Science and Technology,
    University of Chinese Academy of Sciences; Beijing, China}

  \icmlkeywords{quantized neural networks, training complexity, parameterized
    complexity, W[1]-hardness, XNLP, disjoint paths, reachability}

  \vskip 0.3in
]

\printAffiliationsAndNotice{\icmlEqualContribution}

\begin{abstract}
We study the computational complexity of training binary quantized neural
networks, with weights and biases in $\{0,1\}$, on a given architecture,
parameterized by the input dimension $\alpha$ and the output dimension
$\omega$. A small input dimension limits the functions that a single neuron can
compute, and a small output dimension limits the number of target functions
that must be realized simultaneously, but it is unclear whether these
restrictions make training tractable. We prove that training is \Wone-hard parameterized by
$\alpha+\omega$, even for exact fitting of chain-ordered data and with all
biases fixed to zero. To locate the boundary between hardness and
tractability, we characterize realizable networks in three settings:
chain-ordered data, where trainable biases provably add no fitting power;
bias-free networks; and networks with at most two inputs. In each setting,
every attainable vector of predictions is realized by a network supported on
at most $\omega$, $\alpha\omega$, or $2\omega$ input--output paths,
respectively. These characterizations yield
algorithms that run in polynomial time for fixed dimensions and place the
problems in \XNLP. For chain-ordered data, only the output dimension needs to
be fixed. The results identify the global coordination of paths within the
architecture as a computational obstacle that persists even when each neuron
has only a few possible local behaviors.
\end{abstract}

\section{Introduction}
\label{sec:intro}

Given a network architecture and a training set, how hard is it to find
weights and biases that fit the data? When the parameters range over a small
finite set, as in binarized and other low-precision networks
\citep{courbariaux2015binaryconnect,hubara2016bnn}, this is a finite
combinatorial problem: a candidate network has polynomial size and can be
checked exactly, so the problem lies in \NP, whereas training with
real-valued parameters can be $\exists\mathbb{R}$-complete
\citep{abrahamsen2021training,bertschinger2023fully}. We study the binary model
of \citet{ganian2026tractability}, in which weights and biases lie in
$\{0,1\}$. Unlike the binarized networks cited above, whose weights are
$\pm1$, this model has nonnegative weights. Even severely restricted networks of this kind
are NP-hard to train \citep{ganian2026tractability}, so the question is
which features of an instance make training tractable.

A natural candidate is low dimensionality. With $\alpha$ binary inputs, each
neuron computes one of at most $2^{2^\alpha}$ Boolean functions, so a small
input dimension severely limits what a single neuron can do, and a small
output dimension $\omega$ limits how many target functions must be realized
simultaneously. These bounds are local, however. They place no restriction on
the size or connectivity of the architecture, and the number of possible combinations of local behaviors may
still grow exponentially with the number of neurons. Low dimensions
might allow us to avoid searching this space, either in time
$f(\alpha+\omega)\cdot|I|^{O(1)}$ (fixed-parameter tractability, FPT) or at
least in polynomial time for every fixed $(\alpha,\omega)$; alternatively, the
coordination of neurons through the architecture might remain an unavoidable
source of hardness. The central question of this paper is therefore:
\begin{quote}
\itshape
When do low input and output dimensions make the training of binary quantized
neural networks tractable, and when does the architecture itself remain a
source of computational hardness?
\end{quote}

Existing results do not settle this question. \citet{ganian2026tractability}
initiated the parameterized study of quantized training. In their model,
binary training (\twoQNNT) is NP-hard already with one output, and also with
two inputs. Together with the
treewidth $\tw(G)$ of the architecture, however, the two dimensions give
fixed-parameter tractability. Treewidth enters exactly at the point raised
above: it bounds how many neurons must be coordinated across a separator. What
the dimensions alone achieve was left open, and it is this question that we
study for binary networks.

We first show that bounding the input and output dimensions does not suffice
for fixed-parameter tractability. \twoQNNT is \Wone-hard parameterized by
$\alpha+\omega$, even for exact fitting of a prefix chain of examples whose
inputs equal their targets, and even when all biases are fixed to zero
(\cref{thm:main}). On these instances, the hardness comes from routing. When
the training inputs are totally ordered coordinatewise, biases add no fitting
power on the data, and some optimal network has every output copy a single
input or compute $0$ (\cref{thm:chain}). Training then amounts to finding
input--output paths that are vertex-disjoint whenever they start at different
inputs (\cref{thm:routing}). On the prefix chain this is vertex-disjoint
routing between prescribed input--output pairs, the \emph{one-flip routing
equivalence} of \cref{thm:oneflip}, and that problem is \Wone-hard in DAGs
\citep{slivkins2010edp}.

\Wone-hardness still leaves room for algorithms that are polynomial for each
fixed dimension, and we obtain such algorithms in three settings, for
arbitrary architectures and error bounds. The routing characterization itself
gives the first one. On chain-ordered data, training is polynomial for every
fixed $\omega$, whatever the input dimension (\cref{thm:chain-alg}), and with a
single input it is polynomial for every $\omega$ (\cref{cor:one-input}).

Without biases the data need not be ordered. Every output then computes the
disjunction of the inputs that reach it, so training asks which input--output
reachability relations the architecture realizes exactly. Paths chosen
independently for the required pairs may together create forbidden
connections, and we show that realizability is captured exactly once source
sets are propagated globally along the paths (\cref{thm:reach}). This gives
polynomial time for every fixed $(\alpha,\omega)$ (\cref{thm:bf-alg}).

With trainable biases and arbitrary data, we can handle two inputs. A normal
form shows that every realizable output behavior is realized on at most
$2\omega$ input--output paths (\cref{thm:two-witness}), and training is
polynomial for every fixed $\omega$ (\cref{thm:two-alg}). All of these
algorithms also place the corresponding problems in \XNLP, a parameterized
class contained in \XP (\cref{sec:prelim}); \cref{tab:landscape} lists the
bounds.

The three results rest on the same property. Some solution is supported on a
bounded number of input--output paths, and the paths interact through an
explicit rule (\cref{fig:witness}): disjointness across sources, propagation of
source sets, or propagation of function sets with one conjunction rule. With
boundedly many paths, the classical token technique for routing in DAGs
\citep{fortune1980directed,feldman2006directed} runs in polynomial time for
fixed dimensions. The hardness theorem shows that finding a consistent witness
is \Wone-hard even under the simplest of these rules. Whether general binary
training with three or more inputs is polynomial for fixed dimensions remains
open.

\begin{table}[t]
\caption{Parameterized complexity of \twoQNNT. Each row restricts the
instances as indicated. Lower bounds hold already for $\ell=0$. Upper bounds
hold for arbitrary architectures and error bounds, and for arbitrary data
within the restriction of the row. Since $\XNLP\subseteq\XP$, each \XNLP row is
polynomial-time solvable for every fixed value of its parameter; running times
are given in the cited theorems. \Wone-hard rows are not \FPT unless
$\FPT=\Wone$. $^{a}$NP-hard if $\omega$ is unbounded
\citep{ganian2026tractability}.}
\label{tab:landscape}
\vskip 0.05in
\centering
\small
\setlength{\tabcolsep}{4pt}
\begin{tabular}{@{}llll@{}}
\toprule
Setting & Param. & Lower bound & Upper bound\\
\midrule
General & $\alpha+\omega$ & \Wone-hard (\ref{thm:main}) & \NP; \XP open\\
Bias-free & $\alpha+\omega$ & \Wone-hard (\ref{thm:main}) & \XNLP (\ref{thm:bf-alg})\\
Chain-ordered & $\omega$ & \Wone-hard (\ref{thm:main}) & \XNLP (\ref{thm:chain-alg})\\
One input & --- & --- & \Ptime (\ref{cor:one-input})\\
Two inputs & $\omega$ & open$^{a}$ & \XNLP (\ref{thm:two-alg})\\
\bottomrule
\end{tabular}
\vskip -0.1in
\end{table}

\paragraph{Related work.}
NP-hardness of training small networks goes back to
\citet{judd1988loading,blum1992training}. For real-valued ReLU networks,
\citet{arora2018relu} train one-hidden-layer networks with a single output to
global optimality, for convex losses, in time polynomial in the number of data
points when both the input dimension and the hidden-layer width are fixed. For
two-layer networks, \citet{froese2022dimensionality} prove \Wone-hardness
parameterized by the input dimension, and \citet{froese2023fixed} prove
NP-hardness already for input dimension two
\citep[see also][]{goel2021tight,boob2022complexity,brand2023frontiers}. For
discrete parameters, \citet{doronarad2025discrete} studies the hardness of
training deep networks discretely, and
\citet{ganian2026tractability} introduce the model and the parameterized
questions that we study. On the graph side, disjoint paths in DAGs are
\Wone-hard parameterized by the number of paths
\citep{slivkins2010edp,chitnis2023dag_edp} and solvable in time $n^{O(k)}$ by
moving tokens through a topological order \citep{fortune1980directed}. The
directed Steiner network problem asks for a cheap subgraph that satisfies
prescribed \emph{positive} reachability requirements and is solvable in time
$n^{O(k)}$ by a related token game \citep{feldman2006directed}. Bias-free
training additionally imposes \emph{negative} requirements, namely forbidden
input--output connections, and the source sets of \cref{sec:biasfree} are what
keep track of them. \XNLP was studied by \citet{elberfeld2015space} and named
and developed by \citet{bodlaender2024xnlp}.

\section{Model and Preliminaries}
\label{sec:prelim}

We use the model of \citet{ganian2026tractability}, specialized to binary
parameters. For an integer $k\ge1$, $[k]=\{1,\dots,k\}$. Graphs are directed,
paths are directed, and $\Nin(v)$ is the set of in-neighbors of~$v$. For
$x,y\in\{0,1\}^\alpha$, we write $x\le y$ if $x_i\le y_i$ for all $i$.

\paragraph{Architectures and networks.}
An \emph{architecture} is a DAG $G=(V,E)$ with a partition of $V$ into layers
$V_0,\dots,V_L$ such that $V_0$ is the set of sources, every arc goes from some
$V_h$ to $V_{h+1}$, and all sinks lie in one common layer. The sources, in a
fixed order, are the \emph{inputs} $s_1,\dots,s_\alpha$; the sinks, in a fixed
order, are the \emph{outputs} $t_1,\dots,t_\omega$. We assume $L\ge1$ (for
$L=0$ there is nothing to train). Then the outputs are exactly the vertices of
$V_L$, every vertex outside $V_0$ has an in-neighbor, and every input--output
path has exactly one vertex in each layer (\cref{fact:layers}). We write
$n=|V|$ and $|I|$ for the encoding length of an instance.

A \emph{binary network} on $G$ consists of weights $w:E\to\{0,1\}$ and biases
$b:V\setminus V_0\to\{0,1\}$. On $x\in\{0,1\}^\alpha$, input $s_i$ has
activation $x_i$, and every other vertex $v$ has activation
\begin{equation}
\label{eq:activation}
\begin{aligned}
a_v(x)&=\relu_2\Bigl(\textstyle\sum_{u\in\Nin(v)}w(uv)\,a_u(x)-b(v)\Bigr)\\
&=\ind\Bigl[\textstyle\sum_{u\in\Nin(v)}w(uv)\,a_u(x)\ge1+b(v)\Bigr],
\end{aligned}
\end{equation}
where $\relu_2(z)=\min\{\max\{z,0\},1\}$; as in \citet{ganian2026tractability},
the bias is subtracted. Arcs of weight one are \emph{selected}, and the
tails of selected arcs into $v$ are the \emph{selected parents} of $v$. The
network computes $N:\{0,1\}^\alpha\to\{0,1\}^\omega$,
$N(x)=(a_{t_1}(x),\dots,a_{t_\omega}(x))$.

Since $b(v)\in\{0,1\}$, every non-input neuron is either an \emph{OR gate}
(bias $0$: it fires if some selected parent fires) or a \emph{threshold-two
gate} (bias $1$: it fires if at least two selected parents fire). An OR gate
with a single selected parent copies it. The structural results of this paper all describe when threshold-two
gates can be dispensed with or confined to a simple pattern.

A training set is a multiset $\calD$ of examples
$(x,y)\in\{0,1\}^\alpha\times\{0,1\}^\omega$, listed explicitly; repeated and
conflicting examples are allowed. The error $\err_\calD(N)$ is the number of
examples, with multiplicity, for which $N(x)\neq y$. A wrong output coordinate
makes the whole example wrong.

\problembox{Binary Quantized Neural Network Training (\twoQNNT)}
{an architecture $G$, a training multiset $\calD$, and an integer $\ell\ge0$.}
{Is there a binary network $N$ on $G$ with $\err_\calD(N)\le\ell$?}

In \bfQNNT all biases are fixed to zero. Unless stated otherwise, the parameter
is $\alpha+\omega$. Both problems are in \NP: a network is a certificate of
polynomial size that can be checked by evaluating all examples. Two facts about
binary networks are used throughout.

\begin{lemma}[Monotone, zero-preserving neurons]
\label{lem:monotone}
In every binary network, every vertex computes a monotone Boolean function $f$
of $x\in\{0,1\}^\alpha$ with $f(0)=0$.
\end{lemma}

\begin{lemma}[Disjunctive semantics]
\label{lem:disjunction}
In a bias-free binary network, every vertex $v$ computes
$a_v(x)=\bigvee_{i\in A_v}x_i$, where $A_v$ is the set of inputs $s_i$ from
which $v$ is reachable along selected arcs.
\end{lemma}

\begin{proof}[Proof of \cref{lem:monotone,lem:disjunction}]
Use induction over the layers. Inputs compute coordinate projections. If all
in-neighbors of $v$ compute monotone functions vanishing at $0$, then
$\sum_uw(uv)a_u(x)$ is monotone, since weights are nonnegative, and vanishes at
$x=0$; comparing it with the threshold $1+b(v)\ge1$ preserves both properties.
Without biases, $a_v=\bigvee_ua_u$ over the selected parents $u$, and $A_v$ is
the union of their sets $A_u$.
\end{proof}

\paragraph{Parameterized complexity.}
We follow \citet{cygan2015parameterized,downey2013fundamentals}. A
parameterized problem with parameter $\kappa$ is \emph{fixed-parameter
tractable} (FPT) if it is solvable in time $f(\kappa)|I|^{O(1)}$, and in \XP if
it is solvable in time $|I|^{f(\kappa)}$, for a computable $f$. \Wone-hard
problems are not FPT unless $\FPT=\Wone$; hardness transfers along
parameterized reductions. A problem is in \XNLP if a nondeterministic algorithm
decides it in time $f(\kappa)|I|^{O(1)}$ using $f(\kappa)\log|I|$ bits of work
space \citep{elberfeld2015space,bodlaender2024xnlp}. Since such a machine has
$|I|^{O(f(\kappa))}$ configurations, $\XNLP\subseteq\XP$. An \XP algorithm is a single algorithm with a running-time bound $|I|^{f(\kappa)}$; a separate polynomial-time algorithm for each parameter value does not by itself give membership in \XP. A problem is
\emph{\paraNP-hard} if it is NP-hard for some fixed value of the parameter; such
a problem is not in \XP unless $\Ptime=\NP$. Membership in \XNLP together with
\Wone-hardness determines neither \Wone- nor \XNLP-completeness.

\section{Chain-Ordered Data: Fitting Is Routing}
\label{sec:chain}

\begin{figure*}[t]
\centering
\begin{tikzpicture}[
  >=Stealth, x=1.05cm, y=0.62cm,
  v/.style={circle,draw,minimum size=4.6mm,inner sep=0pt,font=\scriptsize},
  lab/.style={font=\scriptsize,align=center},
  ann/.style={font=\scriptsize,inner sep=1pt},
  p1/.style={->,thick},
  p2/.style={->,thick,dashed},
  p3/.style={->,thick,densely dotted}
]
\begin{scope}[xshift=0cm]
  \node[lab] at (1.5,3.15) {(a) chain-ordered data};
  \node[v] (a1) at (0,2) {$s_1$};
  \node[v] (a2) at (0,0) {$s_2$};
  \node[v] (ah1) at (1,2) {};
  \node[v] (ah2) at (1,0) {};
  \node[v] (ag1) at (2,2.4) {};
  \node[v] (ag2) at (2,1.2) {};
  \node[v] (ag3) at (2,0) {};
  \node[v] (at1) at (3,2.4) {$t_1$};
  \node[v] (at2) at (3,1.2) {$t_2$};
  \node[v] (at3) at (3,0) {$t_3$};
  \draw[p1] (a1)--(ah1); \draw[p1] (ah1)--(ag1); \draw[p1] (ah1)--(ag2);
  \draw[p1] (ag1)--(at1); \draw[p1] (ag2)--(at2);
  \draw[p2] (a2)--(ah2); \draw[p2] (ah2)--(ag3); \draw[p2] (ag3)--(at3);
  \node[lab] at (1.5,-1.05) {$\sigma=(1,1,2)$: paths from different\\inputs are vertex-disjoint};
\end{scope}
\begin{scope}[xshift=5.3cm]
  \node[lab] at (1.5,3.15) {(b) bias-free, arbitrary data};
  \node[v] (b1) at (0,2.4) {$s_1$};
  \node[v] (b2) at (0,1.2) {$s_2$};
  \node[v] (b3) at (0,0) {$s_3$};
  \node[v] (bu) at (1,1.8) {$u$};
  \node[v] (bc) at (1,0) {};
  \node[v] (ba) at (2,2.4) {};
  \node[v] (bv) at (2,0.6) {$v$};
  \node[v] (bt1) at (3,2.4) {$t_1$};
  \node[v] (bt2) at (3,1.2) {$t_2$};
  \node[v] (bt3) at (3,0) {$t_3$};
  \draw[p1] (b1)--(bu); \draw[p1] (bu)--(ba); \draw[p1] (ba)--(bt1);
  \draw[p2] (b2)--(bu); \draw[p2] (bu)--(bv); \draw[p2] (bv)--(bt2);
  \draw[p3] (b3)--(bc); \draw[p3] (bc)--(bv); \draw[p3] (bv)--(bt3);
  \node[ann,anchor=south] at ([yshift=1pt]bu.north) {$\{1,2\}$};
  \node[ann,anchor=south] at ([xshift=7pt,yshift=2pt]bv.north) {$\{1,2,3\}$};
  \node[lab] at (1.5,-1.05) {$(1,3)\notin R$, yet $s_1$ reaches $t_3$\\through $u$ and then $v$};
\end{scope}
\begin{scope}[xshift=10.6cm]
  \node[lab] at (1.5,3.15) {(c) two inputs, trainable biases};
  \node[v] (c1) at (0,2) {$s_1$};
  \node[v] (c2) at (0,0) {$s_2$};
  \node[v] (cp) at (1,2) {};
  \node[v] (cq) at (1,0) {};
  \node[v,very thick] (cg) at (2,1) {$g$};
  \node[v] (cr) at (2,2.4) {};
  \node[v] (ct1) at (3,1) {$t_1$};
  \node[v] (ct2) at (3,2.4) {$t_2$};
  \draw[p1] (c1)--(cp); \draw[p1] (cp)--(cg);
  \draw[p2] (c2)--(cq); \draw[p2] (cq)--(cg);
  \draw[->,thick] (cg)--(ct1);
  \draw[p1] (cp)--(cr); \draw[p1] (cr)--(ct2);
  \node[ann,anchor=north west] at ([xshift=-2pt,yshift=-1pt]cg.south) {$x_1\land x_2$, $b=1$};
  \node[ann,anchor=north] at ([yshift=-1pt]cp.south) {$x_1$};
  \node[ann,anchor=south] at ([yshift=1pt]cq.north) {$x_2$};
  \node[lab] at (1.5,-1.05) {$t_1=x_1\land x_2$, $t_2=x_1$;\\the shared vertex computes $x_1$ for both};
\end{scope}
\end{tikzpicture}
\caption{Three kinds of path witnesses. (a)~On chain-ordered data, some optimal
network has every output copy one input or compute $0$, and fitting is
equivalent to routing with paths that are vertex-disjoint across different
inputs (\cref{thm:routing}). (b)~Without biases, an output computes the
disjunction of the inputs that reach it. Here $R$ consists of the diagonal and
$(1,2),(2,1),(2,3),(3,2)$, but not $(1,3)$. Only the three diagonal paths are
drawn; each of the four other pairs has a witness path with private internal
vertices, which is not drawn. Any two paths of this system that meet satisfy
the pairwise condition of \cref{sec:biasfree}, yet $s_1$ reaches $t_3$ through
$u$ and then $v$. The propagated source sets (shown at $u$ and $v$) detect
this (\cref{thm:reach}). (c)~With two inputs and trainable biases, each
conjunction or disjunction output is formed at a single generation vertex~$g$
from copies of the inputs, so $2\omega$ paths suffice
(\cref{thm:two-witness}).}
\label{fig:witness}
\end{figure*}

We first show that on chain-ordered data the extra expressive power of biases
disappears and fitting becomes a routing problem. The hardness theorem and an
algorithm both follow from this routing problem.

\subsection{Biases Are Redundant on Chains}

Let the distinct inputs occurring in $\calD$ be
$x^{(1)}<x^{(2)}<\dots<x^{(m)}$ in the coordinatewise order; we then call
$\calD$ \emph{chain-ordered}. Fix a binary network. By \cref{lem:monotone},
each vertex $v$ is inactive on an initial segment of the chain and active
afterwards, so its \emph{profile} $(a_v(x^{(1)}),\dots,a_v(x^{(m)}))$ is
determined by its \emph{onset}
$\onset_v=\min\{r:a_v(x^{(r)})=1\}$, with $\onset_v=\infty$ if $v$ is never
active.

A neuron switches on when its weighted sum crosses its threshold. Parents that
are already active contribute to the sum both before and after this moment,
so the increase must come from a parent that switches on at the same moment.
This holds for both gate types.

\begin{lemma}[Onset inheritance]
\label{lem:onset}
If $\onset_v<\infty$ for a non-input vertex $v$, then some selected parent $u$
of $v$ satisfies $\onset_u=\onset_v$.
\end{lemma}

\begin{proof}
Let $c(r)$ be the number of selected parents $u$ with $\onset_u\le r$, for
$0\le r\le m$; so $c(0)=0$ and $c$ is nondecreasing. By~\eqref{eq:activation},
$a_v(x^{(r)})=\ind[c(r)\ge1+b(v)]$ for $r\ge1$. For $\onset=\onset_v$ we have
$c(\onset)\ge1+b(v)>c(\onset-1)$, so some selected parent has onset exactly
$\onset$.
\end{proof}

\begin{theorem}[Biases are redundant on chains]
\label{thm:chain}
Let $\calD$ be chain-ordered. For every binary network $(w,b)$ on $G$ there is
a bias-free network $(w',0)$ on $G$ with $w'\le w$, in which every non-input
vertex has at most one selected parent and every vertex has the same profile on
$\calD$ as in $(w,b)$. Hence, on every architecture, \twoQNNT and \bfQNNT have
the same attainable predictions on $\calD$ and the same minimum error.
\end{theorem}

\begin{proof}
Process the non-input vertices layer by layer. If $\onset_v=\infty$, give $v$
no selected parent. Otherwise keep exactly one selected arc $uv$ with
$\onset_u=\onset_v$, which exists by \cref{lem:onset}, and set $b(v)=0$. By
induction over the layers, the profile of $u$ is unchanged, and $v$ now copies
it, so the profile of $v$ is unchanged too. Errors depend only on the
predictions on the distinct inputs.
\end{proof}

The theorem applies to every chain-ordered multiset, with any labels and error
bound. In particular, the hardness theorem below does not rely on the extra
power of biases, and any difference in fitting power between the two models
must come from incomparable inputs (\cref{sec:biases}).

\subsection{Source-Disjoint Routing}

In the network $(w',0)$ of \cref{thm:chain}, the selected arcs have in-degree
at most one. Following them backwards from an output either reaches a single
input $s_i$, and then the output computes $x_i$, or stops at a vertex without
selected parents, and then the output computes $0$. We record this by a
labeling $\sigma:[\omega]\to[\alpha]\cup\{\bot\}$ and the function
$F_\sigma(x)_j=x_{\sigma(j)}$, with $x_\bot=0$.

\begin{definition}[Source-disjoint routing]
\label{def:routing}
A \emph{$\sigma$-routing} in $G$ is a family of paths
$P_j:s_{\sigma(j)}\reach t_j$, one for each $j$ with $\sigma(j)\neq\bot$, such
that $P_j$ and $P_{j'}$ are vertex-disjoint whenever
$\sigma(j)\neq\sigma(j')$.
\end{definition}

Paths from the same input may share vertices (\cref{fig:witness}a).

\begin{theorem}[Fitting chain-ordered data is routing]
\label{thm:routing}
Let $\calD$ be chain-ordered. The minimum of $\err_\calD$ over all binary
networks on $G$, with or without biases, equals
\[
\min\bigl\{\err_\calD(F_\sigma):G\text{ admits a }\sigma\text{-routing}\bigr\}.
\]
\end{theorem}

\begin{proof}
Given a $\sigma$-routing, select its arcs and set all biases to zero. A vertex
on a path from $s_i$ is reachable only from $s_i$, because paths from other
inputs avoid it, so by \cref{lem:disjunction} $t_j$ computes $x_{\sigma(j)}$; an
output with $\sigma(j)=\bot$ has no selected in-arc and computes $0$.
Conversely, convert an optimal network by \cref{thm:chain}. In the result, the
backward walks along unique selected parents define $\sigma$, the network
computes exactly $F_\sigma$, and walks that end at different inputs share no
vertex.
\end{proof}

For the hardness proof we use one specific chain. The architecture has inputs
$s_0,\dots,s_k$ and outputs $t_0,\dots,t_k$, where $s_0$ and $t_0$ are dummies
needed only to make the constructed architectures valid.

\begin{definition}[One-flip data set]
\label{def:oneflip}
For $0\le r\le k$, let $\xi^{(r)}\in\{0,1\}^{k+1}$ have $\xi^{(r)}_0=0$ and
$\xi^{(r)}_i=\ind[i\le r]$ for $i\in[k]$. Let
$\calD_k=\{(\xi^{(r)},\xi^{(r)}):0\le r\le k\}$.
\end{definition}

Each input equals its target, and passing from $\xi^{(r-1)}$ to $\xi^{(r)}$
switches on exactly input $s_r$ and target $t_r$.

\begin{theorem}[One-flip routing equivalence]
\label{thm:oneflip}
Let $G$ have inputs $s_0,\dots,s_k$ and outputs $t_0,\dots,t_k$. Some binary
network on $G$ fits $\calD_k$ exactly if and only if $G$ contains pairwise
vertex-disjoint paths $s_i\reach t_i$ for all $i\in[k]$; in that case a
bias-free network fits $\calD_k$ exactly. Moreover, every vertex of every
binary network has one of only $k+1$ profiles on $\calD_k$.
\end{theorem}

\begin{proof}
Index the chain by $r=0,\dots,k$. The onsets of $s_0,s_1,\dots,s_k$ are
$\infty,1,\dots,k$, which are pairwise distinct. Hence, for $j\in[k]$,
coordinate $j$ of $F_\sigma(\xi^{(r)})$ equals $\ind[j\le r]$ for all $r$
exactly when $\sigma(j)=j$, and coordinate $0$ vanishes for all $r$ exactly
when $\sigma(0)\in\{0,\bot\}$. A $\sigma$-routing for such a $\sigma$ contains
disjoint paths $s_i\reach t_i$, $i\in[k]$, since their sources differ;
conversely, such paths form a $\sigma$-routing with $\sigma(0)=\bot$. Apply
\cref{thm:routing} with $\ell=0$. Profiles are determined by onsets in
$\{1,\dots,k,\infty\}$, since $\xi^{(0)}=0$ forces every vertex to be inactive
at $r=0$.
\end{proof}

Every neuron thus has one of $k+1$ monotone profiles; the difficulty lies in
assigning them to neurons consistently with the architecture
(\cref{ex:bottleneck}). \Cref{app:oneflip} gives a direct proof by predecessor
tracing.

\subsection{W[1]-Hardness}

We reduce from edge-disjoint paths in DAGs.

\problembox{\DAGEDP}
{a DAG $H$ and terminal pairs $(p_1,q_1),\dots,(p_k,q_k)$.}
{Are there pairwise arc-disjoint paths $p_i\reach q_i$, $i\in[k]$?}

\DAGEDP is \Wone-hard parameterized by $k$ \citep{slivkins2010edp}, even on
planar DAGs \citep{chitnis2023dag_edp}. A reduction from it gives our main
lower bound.

\begin{theorem}[Main hardness theorem]
\label{thm:main}
\twoQNNT and \bfQNNT are \Wone-hard parameterized by $\alpha+\omega$, and also
by $\omega$ alone. This holds even on instances with $\ell=0$, training set
$\calD_k$, and $|\calD|=\alpha=\omega=k+1$. Consequently, unless
$\FPT=\Wone$, neither problem is solvable in time
$f(\alpha+\omega)\cdot|I|^{O(1)}$ for a computable $f$.
\end{theorem}

By \cref{thm:oneflip}, it suffices to
produce, in polynomial time, a valid architecture with inputs $s_0,\dots,s_k$
and outputs $t_0,\dots,t_k$ that has vertex-disjoint paths $s_i\reach t_i$ for
all $i\in[k]$ exactly when the \DAGEDP instance is a yes-instance. Three
standard steps do this; \cref{app:reduction} contains complete proofs.

\emph{Terminal stubs.} Attach private arcs $p_i^-\to p_i$ and
$q_i\to q_i^+$ to every pair and use $(p_i^-,q_i^+)$ as the new terminals. The
answer is unchanged, the terminals become distinct, and every solution path has
at least two arcs.

\emph{Line graph.} Let $K$ have a vertex $z_e$ for every arc $e$ of $H$, an arc
$z_e\to z_f$ whenever $\head(e)=\tail(f)$, a fresh source $s_i$ with arcs to all
$z_e$ with $\tail(e)=p_i$, and a fresh sink $\hat t_i$ with arcs from all $z_e$
with $\head(e)=q_i$. Then arc-disjoint $p_i\reach q_i$ paths in $H$ correspond
to vertex-disjoint $s_i\reach\hat t_i$ paths in $K$ (\cref{lem:linegraph}).

\emph{Normalization.} Add a dummy input $s_0$ with arcs to all other sources of
$K$, a dummy output $\hat t_0$ with arcs from all other sinks, and the arc
$s_0\hat t_0$. Place each vertex at the level given by the longest path ending
at it, subdivide arcs that skip levels, and extend every sink by a private path
to the last level. The result is an architecture $G$ with inputs
$s_0,\dots,s_k$ and outputs $t_0,\dots,t_k$ whose vertex-disjoint
$s_i\reach t_i$ paths, $i\in[k]$, correspond to those of $K$
(\cref{lem:normalization}), because subdivision and padding paths are
private.

\begin{proof}[Proof of \cref{thm:main}]
We combine the three steps with \cref{thm:oneflip}. The reduction maps
$(H,(p_i,q_i)_{i\in[k]})$ to $(G,\calD_k,0)$ in polynomial
time, with $\alpha=\omega=k+1$. By \cref{lem:linegraph,lem:normalization} and
\cref{thm:oneflip}, it preserves the answer for both models, since a yes-instance
admits a bias-free solution and a no-instance admits no solution even with
biases. It is therefore a parameterized reduction with parameter map $k\mapsto
2k+2$, and also $k\mapsto k+1$ for the parameter~$\omega$.
\end{proof}

On the produced instances, the depth, width, and treewidth of $G$ are not
bounded by any function of $\alpha+\omega$. For width and treewidth this is
unavoidable unless $\FPT=\Wone$, since either one together with $\alpha+\omega$
gives fixed-parameter tractability \citep{ganian2026tractability}.

\subsection{Algorithms for Chain-Ordered Data}

\Cref{thm:routing} reduces chain-ordered training to at most $\omega$ paths
whose only interaction is disjointness across sources. Disjoint paths in DAGs
are classically found by moving one token per path through the layers
\citep{fortune1980directed}; we use the same technique.

\begin{theorem}[Upper bounds on chain-ordered data]
\label{thm:chain-alg}
On chain-ordered instances, \twoQNNT and \bfQNNT are solvable in time
$2^{O(\omega\log\omega)}\cdot|I|^{\omega+O(1)}$ and belong to \XNLP
parameterized by $\omega$.
\end{theorem}

\begin{proof}[Proof sketch]
Compute $\err_\calD(F_\sigma)$ for each of the at most $(\alpha+1)^\omega$
labelings $\sigma$. For every $\sigma$ with error at most $\ell$, place a token
for each $j$ with $\sigma(j)\neq\bot$ on $s_{\sigma(j)}$. Tokens move through
the layers one at a time along arcs. Tokens with different sources may never
occupy the same vertex. We accept if every token reaches its output.

Enumerating labelings and then searching separately for each would cost
$|I|^{2\omega+O(1)}$. Instead, we run a single search started from all
labelings whose error is at most $\ell$. Once the tokens are placed, their
sources matter only through the partition of tokens into classes with a
common source. A configuration is therefore at most $\omega$ positions and one
such partition, which gives $2^{O(\omega\log\omega)}n^{\omega}$ configurations
per stage. A nondeterministic algorithm guesses $\sigma$ and the moves and
stores one configuration in $O(\omega\log|I|)$ bits. Correctness follows from
\cref{thm:routing}; details are in \cref{app:chain-alg}.
\end{proof}

For one input, every training set is chain-ordered. A single source imposes no
disjointness constraint, so a $\sigma$-routing exists if and only if every
output $t_j$ with $\sigma(j)=1$ is reachable from $s_1$ in $G$. The best
$\sigma$ can be found by counting labels.

\begin{corollary}[One input]
\label{cor:one-input}
\twoQNNT with $\alpha=1$ is solvable in polynomial time, for arbitrary
$\omega$, data, and error bounds.
\end{corollary}

\section{Bias-Free Training: Exact Reachability}
\label{sec:biasfree}

We now drop the chain assumption but fix all biases to zero. By
\cref{lem:disjunction}, the function computed by a bias-free network depends
only on which inputs reach which outputs along selected arcs. For a relation
$R\subseteq[\alpha]\times[\omega]$, let $C_j=\{i:(i,j)\in R\}$ and
$F_R(x)_j=\bigvee_{i\in C_j}x_i$. Call $R$ \emph{realizable} on $G$ if some set
$H\subseteq E$ of selected arcs has input--output reachability relation exactly
$R$. Since $F_R$ determines $R$, \cref{lem:disjunction} gives:

\begin{observation}
\label{obs:bf}
A \bfQNNT instance $(G,\calD,\ell)$ is a yes-instance if and only if some
relation $R$ that is realizable on $G$ satisfies $\err_\calD(F_R)\le\ell$.
\end{observation}

The data enter only through the test $\err_\calD(F_R)\le\ell$, and they
typically allow many relations. For
$\calD_k$ and $\ell=0$, for instance, the admissible relations are exactly those
with $\min(C_j\cap[k])=j$ for $j\in[k]$ and $C_0\cap[k]=\emptyset$; a column
$C_j$ may contain any $i>j$, and $C_j$ need not equal $\{j\}$. What
\cref{thm:oneflip} shows is that every \emph{realizable} admissible relation
forces disjoint diagonal paths: if $s_i\reach t_i$ and $s_j\reach t_j$ with
$i<j$ met, then $i$ would reach $t_j$. The general question is which relations
are realizable.

A natural approach is to choose one $s_i\reach t_j$ path for each required
pair $(i,j)\in R$. This guarantees all required connections, but the union of
the paths may introduce additional ones. If two
chosen paths for $(i,j)$ and $(i',j')$ share a vertex, then $i$ reaches $t_{j'}$
and $i'$ reaches $t_j$. Requiring $(i,j'),(i',j)\in R$ for every two chosen paths that
meet is necessary but not sufficient, because sources also travel through
chains of intersections. In \cref{fig:witness}b, let $R$ consist of the
diagonal and the pairs $(1,2),(2,1),(2,3),(3,2)$. Choose the three drawn
diagonal paths, and let each off-diagonal pair use a witness path whose
internal vertices are private (these paths are not drawn). Two paths of this
system meet either only at a common endpoint, where the condition holds
trivially, or at $u$ or $v$, where the crossed pairs lie in $R$. Nevertheless
$s_1$ reaches $t_3$ through $u$ and then $v$, although $(1,3)\notin R$.
Realizability is a global property of the path system.

To account for such indirect connections, we attach to every path a set of
input sources at each layer. An \emph{$R$-path system} is a family $\calP=(P_d)_{d\in R}$ in which
$P_{(i,j)}$ is an $s_i\reach t_j$ path. Write $v_h(d)$ for the vertex of $P_d$
in layer $V_h$. Whenever several paths enter the same vertex, the vertex inherits every source
carried by any of them, and all paths leaving it carry this common set.
Formally, the \emph{source sets} are $S_0(i,j)=\{i\}$ and
\begin{equation}
\label{eq:propagate}
S_{h+1}(d)=\bigcup\bigl\{S_h(e):e\in R,\ v_{h+1}(e)=v_{h+1}(d)\bigr\}.
\end{equation}
By induction over~$h$, $S_h(d)$ is exactly the set of inputs that reach
$v_h(d)$ in the union $H_\calP$ of the arcs of $\calP$: the arcs of $H_\calP$
entering $v_{h+1}(d)$ are those used by the paths through it, and a repeated
arc contributes the same set again.

\begin{theorem}[Exact reachability realization]
\label{thm:reach}
Let $R\subseteq[\alpha]\times[\omega]$. The following are equivalent.
\begin{enumerate}[label=(\roman*),nosep,leftmargin=*]
\item $R$ is realizable on $G$.
\item Some $R$-path system satisfies $S_L(i,j)=C_j$ for all $(i,j)\in R$.
\item Some $R$-path system satisfies $S_h(i,j)\subseteq C_j$ for every layer
  $h$ and all $(i,j)\in R$.
\end{enumerate}
The union of the paths of any system as in (ii) realizes $R$. In particular,
every realizable relation is realized by the union of at most $\alpha\omega$
input--output paths.
\end{theorem}

\begin{proof}
(i)$\Rightarrow$(ii): Let $H$ realize $R$ and pick one $s_i\reach t_j$ path in
$H$ for each $(i,j)\in R$. The union $H_\calP\subseteq H$ contains all required
connections and creates no new ones, so it realizes $R$, and $S_L(i,j)$, the
set of inputs reaching $t_j$ in $H_\calP$, equals $C_j$.
(ii)$\Rightarrow$(i): $H_\calP$ gives every output with $C_j\neq\emptyset$
exactly the sources $C_j$. An output with $C_j=\emptyset$ ends no path, and no
path passes through it since outputs are sinks, so nothing reaches it.
(ii)$\Rightarrow$(iii): If $i'\in S_h(i,j)$, then $i'$ reaches $v_h(i,j)$ in
$H_\calP$, and the rest of $P_{(i,j)}$ leads to $t_j$, so $i'\in S_L(i,j)=C_j$.
(iii)$\Rightarrow$(ii): (iii) gives $S_L(i,j)\subseteq C_j$, and for each
$i'\in C_j$ the path $P_{(i',j)}$ puts $i'$ into the source set of $t_j$.
\end{proof}

Thus one path per required pair always suffices for a witness; what fails is
choosing these paths independently. The safety condition (iii) can be
checked layer by layer, which makes the theorem algorithmic. The
witness has at most $\alpha\omega$ vertices per layer, whereas the depth,
width, and treewidth of $G$ remain arbitrary.

\begin{theorem}[Upper bounds for bias-free training]
\label{thm:bf-alg}
\BfQNNT is solvable deterministically in time
$2^{O(\alpha^2\omega)}\cdot|I|^{\alpha\omega+O(1)}$, for arbitrary
architectures, training multisets, and error bounds, and it belongs to \XNLP
parameterized by $\alpha+\omega$. Hence it is in \XP, and polynomial-time
solvable for every fixed $(\alpha,\omega)$.
\end{theorem}

\begin{proof}[Proof sketch]
Enumerate the $2^{\alpha\omega}$ relations $R$ and keep those with
$\err_\calD(F_R)\le\ell$. For each, place a token for every pair
$(i,j)\in R$ on $s_i$, carrying the set $\{i\}$. Process the layers one at a
time, and within a layer move the tokens in a fixed order. When a token with
set $A$ moves along an arc into a vertex $v$, its new set is $A$ together with
the sets of the tokens that have already moved into $v$, and those tokens
receive the new set as well. After all tokens have moved, the sets are exactly
the propagated source sets~\eqref{eq:propagate}. Accept if every token for
$(i,j)$ ends at $t_j$ with set $C_j$. By \cref{thm:reach,obs:bf}, the
algorithm accepts if and only if the instance is a yes-instance. Every
accepting run yields an actual network, namely the union of the token paths.
With $p=|R|\le\alpha\omega$ tokens, there are at most $n^p2^{\alpha p}$
configurations per stage and $pL$ stages. A nondeterministic algorithm guesses
$R$ and the moves and stores one configuration in $O(\alpha\omega(\alpha+\log
n))$ bits. \Cref{app:biasfree} gives the details.
\end{proof}

With \cref{thm:main}, \bfQNNT is \Wone-hard and in $\XNLP\subseteq\XP$
parameterized by $\alpha+\omega$. Fixed dimensions give polynomial-time
solvability, but unless $\FPT=\Wone$ the degree of the polynomial cannot be
bounded independently of the dimensions. Whether the problem is complete for
\Wone, for \XNLP, or for neither remains open.

\Cref{thm:reach} extends \cref{thm:routing}. On chain-ordered data,
\cref{thm:chain} lets us restrict attention to relations with $|C_j|\le1$,
which are the relations of labelings $\sigma$. For such relations, any
intersection of paths from different inputs is a forbidden connection, so the
source-set condition reduces to disjointness across sources, as in
\cref{def:routing}. Arbitrary data may require columns of any size. Intersections
are then allowed but must be tracked globally, and the number of paths grows
from $\omega$ to $\alpha\omega$.

\section{Trainable Biases on Arbitrary Data}
\label{sec:biases}

Off chain-ordered data, threshold-two gates matter. Take two inputs feeding a
single output and the data $\{((1,0),0),((0,1),0),((1,1),1)\}$. With bias $1$
the output computes $x_1\land x_2$ and fits the data. Without biases it
computes $0$, $x_1$, $x_2$, or $x_1\lor x_2$, and none of these fits. As
\cref{thm:chain} requires, the data contain the incomparable inputs $(1,0)$ and
$(0,1)$. \Cref{prop:gap} gives a less trivial example, with
$\alpha=\omega=5$: biases let a network fit data that would otherwise require
five vertex-disjoint paths through four vertices.

A gap in fitting power is not a gap in complexity, and we do not know whether
biases make training harder in general. For two inputs, however, the extra
expressive power of biases does not prevent polynomial-time solvability for
every fixed output dimension.

\subsection{Two Inputs: a Normal Form With Bounded Path Witnesses}

Let $\alpha=2$. By \cref{lem:monotone}, every vertex computes one of the five
zero-preserving monotone functions of two variables,
\[
0,\quad \fA=x_1,\quad \fB=x_2,\quad \fC=x_1\land x_2,\quad \fD=x_1\lor x_2 .
\]
These are full Boolean functions on $\{0,1\}^2$, not profiles on the data, so
the output behavior of a network is a vector in $\{0,\fA,\fB,\fC,\fD\}^\omega$.
Finiteness alone does not give an algorithm: in general, functions that are
attainable at two vertices separately need not be attainable simultaneously
when the two vertices share ancestors (\cref{prop:obstruction}). We need a
bounded witness that keeps the functions of shared vertices consistent.

With two inputs, a conjunction or disjunction can always be generated
directly from copies of the two input variables, and no other intermediate
computation is needed.

\begin{lemma}[Two-input normal form]
\label{lem:two-normal}
Every binary network with $\alpha=2$ can be modified, without changing the
Boolean function computed at any vertex, so that every non-input vertex is of
one of three kinds: a \emph{zero vertex} with no selected parent; a \emph{copy
vertex} with bias $0$ and one selected parent computing the same function; or a
\emph{generation vertex} computing $\fC$ or $\fD$, with exactly two selected
parents, computing $\fA$ and $\fB$, and bias $1$ for $\fC$ and $0$ for $\fD$.
\end{lemma}

\begin{proof}
Process the vertices layer by layer. A vertex computing $0$ loses its selected
arcs. A vertex computing $f$ that has a selected parent computing $f$ keeps only
that arc and gets bias $0$. Otherwise: if $f=\fA$, every selected parent computes
$0$, $\fB$, $\fC$, or $\fD$, each with value at $(1,0)$ at most its value at
$(0,1)$; the weighted sum and the threshold inherit this, contradicting
$\fA(1,0)=1>\fA(0,1)$. The case $f=\fB$ is symmetric. If $f=\fD$, the values at
$(1,0)$ and $(0,1)$ require a selected $\fA$-parent and a selected $\fB$-parent,
and their disjunction is $\fD$. If $f=\fC$, a disjunction of parents in
$\{0,\fA,\fB,\fD\}$ cannot vanish at $(1,0)$ and $(0,1)$ but not at $(1,1)$, so
$b=1$; with $n_\fA,n_\fB,n_\fD$ selected parents of each type, the three inputs
give $n_\fA+n_\fD\le1$, $n_\fB+n_\fD\le1$, and $n_\fA+n_\fB+n_\fD\ge2$, which
force $n_\fD=0$ and $n_\fA=n_\fB=1$. Each step changes only the in-arcs and bias
of the current vertex, whose parents already compute their original functions.
\end{proof}

In normal form, the vertices that an output depends on are easy to describe.
Going backwards from an $\fA$-output gives a chain of $\fA$-copies ending at
$s_1$. Going backwards from a $\fC$- or $\fD$-output gives a chain of copies
ending at a generation vertex, which splits into an $\fA$-chain to $s_1$ and a
$\fB$-chain to $s_2$ (\cref{fig:witness}c). Deselecting all other arcs changes
no function on these chains, since every vertex on them keeps all of its
selected parents.

\begin{theorem}[Bounded path witness]
\label{thm:two-witness}
Let $\alpha=2$. Every output-function vector realizable on $G$ is realized by a
binary network whose selected arcs form the union of at most $2\omega$
input--output paths: one path from $s_1$ or $s_2$ to each $\fA$- or
$\fB$-output, and one path from each input to each $\fC$- or $\fD$-output. In
this network, every non-input vertex on the paths is a copy vertex or a
generation vertex.
\end{theorem}

A vertex may lie on the paths of several outputs, for instance an $\fA$-copy
that feeds both an $\fA$-output and a generation vertex; it then computes one
function for all of them, which is what the algorithm below must preserve.

\begin{theorem}[Upper bounds for two inputs]
\label{thm:two-alg}
\twoQNNT with $\alpha\le2$ is solvable in time
$2^{O(\omega)}\cdot|I|^{2\omega+O(1)}$, for arbitrary architectures, training
multisets, and error bounds, and it belongs to \XNLP parameterized by $\omega$.
\end{theorem}

\emph{Algorithm.} For $\alpha=1$, every training set is chain-ordered, so
\cref{cor:one-input} gives polynomial time and \cref{thm:chain-alg} gives
membership in \XNLP. For $\alpha=2$,
enumerate the $5^\omega$ vectors $(f_1,\dots,f_\omega)$ and keep those whose
predictions have error at most $\ell$. Place one token for each path of
\cref{thm:two-witness}: a token from $s_1$ or $s_2$ for each $\fA$- or
$\fB$-output, and one token from each input for each $\fC$- or $\fD$-output.
A configuration records the position of every token and, for every occupied
vertex, a label. The tokens move layer by layer, one at a time. A vertex $v$ of
the next layer collects the set $\Phi_v$ of functions of the vertices from
which tokens enter it. When the layer is complete, $v$ receives a function by
one of two rules:
\begin{itemize}[leftmargin=*,nosep]
\item bias $0$: $v$ computes the disjunction of the functions in $\Phi_v$;
\item bias $1$, allowed only if $\Phi_v=\{\fA,\fB\}$: $v$ computes $\fC$.
\end{itemize}
The run accepts if every token ends at its output and the function there is
the prescribed one. For example, if tokens from an $\fA$-vertex and a
$\fB$-vertex enter $v$, then $\Phi_v=\{\fA,\fB\}$ and $v$ may compute $\fD$ or
$\fC$; if two tokens enter $v$ from the same $\fA$-vertex, then $\Phi_v=\{\fA\}$
and $v$ copies $\fA$.

One subtlety is that a threshold-two gate depends on how many of its parents
fire, whereas a configuration retains only a set of functions, so it cannot
simulate an arbitrary threshold-two neuron. The normal form of
\cref{lem:two-normal} resolves this: the algorithm searches only for the
structures it allows, namely OR gates, where multiplicities are
irrelevant, and generation vertices with one $\fA$-parent and one
$\fB$-parent. When an accepting run uses bias $1$ at $v$, the network built
from the run selects exactly one arc from an $\fA$-vertex and one arc from a
$\fB$-vertex into $v$, which are different vertices, so $v$ computes $\fC$ as
recorded. Shared arcs and tokens that travel together are harmless for the
same reason. Hence every accepting run yields a network with the prescribed
outputs, and \cref{thm:two-witness} guarantees an accepting run whenever the
vector is realizable. With at most $2\omega$ tokens and $16$ possible labels,
there are $2^{O(\omega)}n^{2\omega}$ configurations per stage, and a
nondeterministic algorithm stores one configuration in $O(\omega\log|I|)$ bits.
\Cref{app:two-alg} gives the full proof.

Together with \cref{cor:one-input} and NP-hardness for $\alpha=2$ and unbounded
$\omega$ \citep{ganian2026tractability}, the boundary for small input
dimension is as follows. One input is polynomial for every $\omega$. Two inputs
are polynomial for every fixed $\omega$ and NP-hard when $\omega$ is
unbounded. Whether two inputs are \FPT or \Wone-hard parameterized by $\omega$
is open.

\subsection{Conjunctions With More Inputs}

For conjunction targets with more than two inputs, an analogous structural
characterization holds: the single generation vertex is replaced by a binary
merge tree.

\begin{theorem}[Realizing a conjunction]
\label{thm:conj}
Let $\omega=1$ and $\emptyset\neq S\subseteq[\alpha]$. Some binary network on
$G$ computes $\bigwedge_{i\in S}x_i$ on all of $\{0,1\}^\alpha$ if and only if
$G$ contains an in-arborescence rooted at $t_1$ whose leaves are exactly the
inputs $s_i$, $i\in S$, and in which every vertex has in-degree at most two.
This can be decided in time $|I|^{O(|S|)}$.
\end{theorem}

For sufficiency, copy at in-degree one and use threshold two at in-degree two.
For necessity, route a backward demand for each $i\in S$ from $t_1$ towards
parents that switch off when $x_i$ is removed; a bias-$1$ vertex may split its
demands between two parents (\cref{app:conj}). The theorem concerns one
conjunction at one output. Merge trees for different conjunctions or different
outputs can interfere: merging a source into one tree may violate a negative
requirement of another output. We have no normal form for general functions
with $\alpha\ge3$.

\section{Discussion and Open Problems}
\label{sec:discussion}

In the three settings studied here, bounded dimensions bound the number of
paths that some solution needs, but not the cost of routing them: bias-free and
chain-ordered training are polynomial for every fixed dimension, but not FPT
unless $\FPT=\Wone$. Several questions remain open.
\begin{enumerate}[leftmargin=*,nosep]
\item \emph{Membership in \XP.} Is \twoQNNT with trainable biases in \XP
  parameterized by $\alpha+\omega$? Even polynomial-time solvability for
  $\alpha=3$ and every fixed $\omega$ is open. A bounded-witness theorem that
  controls interacting merge trees (\cref{thm:conj}) is one possible route;
  local enumeration of functions does not suffice (\cref{prop:obstruction}).
\item \emph{\paraNP-hardness.} Is \twoQNNT NP-hard for some fixed
  $(\alpha_0,\omega_0)$? Unless $\Ptime=\NP$, \cref{thm:two-alg} forces
  $\alpha_0\ge3$, and \cref{thm:chain-alg,thm:bf-alg} force such a reduction to
  produce instances with incomparable inputs on which biases matter
  (see also \cref{prop:support}).
\item \emph{Exact classification.} \BfQNNT and chain-ordered training are
  \Wone-hard and in \XNLP. Are they in \Wone, or are they \XNLP-hard? Is
  two-input training \Wone-hard parameterized by $\omega$?
\item \emph{Signed weights.} For $d\ge3$, negative weights break onset
  inheritance (\cref{prop:signed}), and the complexity parameterized by
  $\alpha+\omega$ is open.
\end{enumerate}

Our results concern the worst-case complexity of empirical-error minimization
over a fixed binary architecture, not gradient-based training or typical data.

\section*{AI Use Statement}
Generative AI tools were used to assist with literature search, with checking
and organizing proofs, with writing exhaustive verification code for small
instances, and with drafting and editing the manuscript. The authors are
responsible for the content of the paper.

\section*{Impact Statement}
This paper presents work whose goal is to advance the field of Machine
Learning. There are many potential societal consequences of our work, none
which we feel must be specifically highlighted here.

\bibliography{references}
\bibliographystyle{icml2027}

\newpage
\appendix
\onecolumn
\section{Additional Preliminaries}
\label{app:prelim}

\paragraph{Graphs.}
A \emph{walk} is a sequence of vertices in which consecutive vertices are joined
by arcs in the forward direction; a \emph{path} is a walk without repeated
vertices. Paths may have length zero. For an arc $e=uv$ we write
$\tail(e)=u$ and $\head(e)=v$. A family of paths is \emph{vertex-disjoint} if no
two of its paths share a vertex, and \emph{arc-disjoint} if no two share an arc.

\begin{fact}
\label{fact:walk}
In a DAG, every walk is a path.
\end{fact}

\begin{proof}
A repeated vertex would close a directed cycle.
\end{proof}

\begin{fact}[Layered architectures]
\label{fact:layers}
Let $G$ be an architecture with layers $V_0,\dots,V_L$, $L\ge1$, and
$V_L\neq\emptyset$. Then:
(a)~every vertex of $V_h$ with $h\ge1$ has an in-neighbor, and all its
in-neighbors lie in $V_{h-1}$;
(b)~the sinks are exactly the vertices of $V_L$, so the outputs are exactly
$V_L$, and no input is an output;
(c)~every path from $V_h$ to $V_{h'}$ has $h'-h$ arcs; in particular, every
input--output path has exactly one vertex in each layer, and no output is an
interior vertex of a path.
\end{fact}

\begin{proof}
(a) Vertices outside $V_0$ are not sources, and arcs join consecutive layers.
(b) A vertex of $V_L$ has no out-arc, so it is a sink. All sinks lie in one
common layer, which therefore is $V_L$. Since $L\ge1$, $V_0\cap V_L=\emptyset$.
(c) Each arc raises the layer index by one, and outputs have no out-arcs.
\end{proof}

\paragraph{Parameterized reductions.}
A \emph{parameterized reduction} from $(P,\kappa)$ to $(Q,\kappa')$ maps an
instance $I$ to an instance $I'$ such that $I\in P\iff I'\in Q$,
$\kappa'(I')\le g(\kappa(I))$ for a computable $g$, and the map runs in time
$f(\kappa(I))|I|^{O(1)}$. If $Q$ is FPT then so is $P$, and if $P$ is
\Wone-hard then so is $Q$ \citep{cygan2015parameterized,downey2013fundamentals}.
Our reduction runs in polynomial time and has $g(k)=2k+2$ (for $\alpha+\omega$)
or $g(k)=k+1$ (for $\omega$).

\paragraph{\XNLP and \XP.}
A nondeterministic machine that runs in time $f(\kappa)|I|^{c}$ and uses
$f(\kappa)\log|I|$ bits of work space has at most
$|I|^{O(1)}\cdot2^{O(f(\kappa)\log|I|)}=|I|^{O(f(\kappa))}$ configurations, counting
the input-head position. Deciding whether an accepting configuration is
reachable from the initial one is a reachability question in this
configuration graph, which takes time polynomial in its size. Hence
$\XNLP\subseteq\XP$.

\section{Proofs for Chain-Ordered Data}
\label{app:chain}

Throughout this appendix, $\calD$ is chain-ordered with distinct inputs
$x^{(1)}<\dots<x^{(m)}$, and onsets and profiles are taken with respect to this
chain.

\subsection{Details for \texorpdfstring{\cref{thm:routing}}{Theorem 3.4}}

\begin{proof}[Proof of \cref{thm:routing}]
Let $M$ be the minimum error over binary networks and $M_0$ the minimum over
bias-free networks; $M\le M_0$ trivially, and $M_0\le M$ by \cref{thm:chain}. Let
$M_\sigma$ denote the minimum in the statement.

$M_0\le M_\sigma$: let $(P_j)$ be a $\sigma$-routing attaining $M_\sigma$.
Select exactly the arcs of the paths and set all biases to zero. Let $v$ lie on
some $P_j$ with $\sigma(j)=i$. Every selected arc entering $v$ lies on a path
$P_{j'}$ through $v$, and $\sigma(j')=i$ because paths from different inputs are
disjoint. Following selected arcs backwards from $v$ therefore stays on paths
from $s_i$ and can only reach the input $s_i$, while $s_i$ does reach $v$. By
\cref{lem:disjunction}, $v$ computes $x_i$, and in particular $t_j$ computes
$x_{\sigma(j)}$. If $\sigma(j)=\bot$, no path ends at $t_j$ and no path passes
through it (\cref{fact:layers}(c)), so $t_j$ has no selected in-arc and computes
$0$. The network computes $F_\sigma$, so its error is $\err_\calD(F_\sigma)=M_\sigma$.

$M_\sigma\le M$: take a binary network with error $M$ and convert it by
\cref{thm:chain} into $(w',0)$ with the same predictions on $\calD$, hence error
$M$. Every non-input vertex $v$ has at most one selected parent, so there is a
unique maximal backward walk from $v$ along selected arcs. It ends either at an
input $s_i$, in which case $v$ computes $x_i$ by \cref{lem:disjunction}, or at a
non-input vertex without selected parents, in which case $v$ computes $0$. Let
$\sigma(j)$ be the index of the input reached from $t_j$, or $\bot$. Then
$(w',0)$ computes $F_\sigma$ on all of $\{0,1\}^\alpha$, so
$\err_\calD(F_\sigma)=M$. The reversed backward walks $P_j$ form a
$\sigma$-routing: if $P_j$ and $P_{j'}$ share a vertex $v$, then both continue
backward from $v$ along the unique walk of $v$, so $\sigma(j)=\sigma(j')$.
\end{proof}

\subsection{The One-Flip Data Set}
\label{app:oneflip}

\Cref{thm:oneflip} was derived from \cref{thm:routing}. For completeness we
also give the direct proof by predecessor tracing, which uses only the chain
$\calD_k$. Let $G$ have inputs $s_0,\dots,s_k$ and outputs $t_0,\dots,t_k$, fix a
binary network, write $a_v(r)$ for the activation of $v$ on $\xi^{(r)}$, and let
$C_i=\{v:a_v(i-1)=0,\ a_v(i)=1\}$ for $i\in[k]$ be the set of vertices that
switch on at transition~$i$.

\begin{claim}[Predecessor tracing]
\label{claim:trace}
(a) $0=a_v(0)\le a_v(1)\le\dots\le a_v(k)$ for every $v$. Hence the sets
$C_1,\dots,C_k$ are pairwise disjoint. (b) Every non-input vertex $v\in C_i$ has
a selected parent in $C_i$. (c) The only input in $C_i$ is $s_i$.
\end{claim}

\begin{proof}
(a) is \cref{lem:monotone}, since $\xi^{(0)}=0$ and the $\xi^{(r)}$ increase; a
nondecreasing $0/1$ sequence switches on at most once. (b) is \cref{lem:onset}:
the onset of $v$ is $i$. Directly, $S_v(r)=\sum_uw(uv)a_u(r)$ satisfies
$S_v(i-1)\le b(v)<S_v(i)$, and
$S_v(i)-S_v(i-1)=\sum_uw(uv)(a_u(i)-a_u(i-1))$ is a sum of terms in $\{0,1\}$,
one of which must equal one. (c) $s_j\in C_i$ means
$\xi^{(i-1)}_j=0<\xi^{(i)}_j=1$, which holds only for $j=i$.
\end{proof}

\begin{proof}[Direct proof of \cref{thm:oneflip}]
If a network fits $\calD_k$ exactly, then $a_{t_i}(i-1)=\xi^{(i-1)}_i=0$ and
$a_{t_i}(i)=1$, so $t_i\in C_i$. By \cref{claim:trace}(b), starting from $t_i\in
V_L$ and repeatedly stepping to a selected parent in $C_i$ yields vertices
$v_L=t_i,v_{L-1},\dots,v_0$ with $v_h\in V_h\cap C_i$; by (c), $v_0=s_i$. This is
an $s_i\reach t_i$ path inside $C_i$, and paths for different $i$ are disjoint
by (a). Conversely, given vertex-disjoint paths $P_i:s_i\reach t_i$, select
exactly their arcs and set all biases to zero. Every vertex of $P_i$ other than
$s_i$ has exactly one selected parent, its predecessor on $P_i$, and every other
non-input vertex has none. By induction, vertices of $P_i$ compute $x_i$ and all
other non-input vertices compute $0$; the output $t_0$ lies on no path. Hence
$t_i$ outputs $\xi^{(r)}_i$ for all $r$ and $i\in\{0,\dots,k\}$. The profile
count follows from (a): a profile is determined by the position of its first
one, if any.
\end{proof}

All $k+1$ profiles occur, since $s_0,\dots,s_k$ realize them. The difficulty is
therefore not the variety of profiles but their placement.

\begin{example}[Two profiles competing for one neuron]
\label{ex:bottleneck}
Let $k=2$ and let $G$ have layers $\{s_0,s_1,s_2\}$, $\{x,h\}$, and
$\{t_0,t_1,t_2\}$ with arcs $s_0x$, $s_1h$, $s_2h$, $xt_0$, $ht_1$, $ht_2$. The
target coordinate $t_1$ alone can be fitted on $\calD_2$ by selecting $s_1h$
and $ht_1$, so that $h$ has profile $(0,1,1)$; symmetrically, $t_2$ alone can be
fitted with $h$ having profile $(0,0,1)$. Both cannot be fitted at once: the
only parent of $t_1$ and $t_2$ is $h$, so \cref{claim:trace}(b) would put $h$
into $C_1\cap C_2=\emptyset$. Equivalently, $G$ has no disjoint
$s_1\reach t_1$ and $s_2\reach t_2$ paths. Exhaustive enumeration of all
$2^{11}$ parameter assignments confirms this (ancillary file
\texttt{anc/sanity\_checks.py}).
\end{example}

\subsection{Algorithms for Chain-Ordered Data}
\label{app:chain-alg}

\begin{proof}[Proof of \cref{thm:chain-alg}]
Whether $\calD$ is chain-ordered can be checked by comparing all pairs of
examples, in polynomial time and logarithmic space. Number the outputs
$1,\dots,\omega$.

\emph{Token game.} For a labeling $\sigma$, let $J_\sigma=\{j:\sigma(j)\neq\bot\}$.
A \emph{configuration} at stage $(h,r)$, with $0\le h<L$ and $0\le r\le\omega$,
consists of a position $v_j$ for every $j\in J_\sigma$ and the partition
$\pi_\sigma$ of $J_\sigma$ into the classes $\sigma^{-1}(i)$; here
$v_j\in V_{h+1}$ for $j\le r$ (moved tokens) and $v_j\in V_h$ for $j>r$. It is
\emph{valid} if tokens in different classes occupy different vertices. The
initial configuration of $\sigma$ at stage $(0,0)$ has $v_j=s_{\sigma(j)}$.
From stage $(h,r-1)$, the token $r$, if $r\in J_\sigma$, moves along an arc
$v_rv'$ to some $v'\in V_{h+1}$, provided that no already moved token of another
class occupies $v'$; tokens $r\notin J_\sigma$ are skipped. Stage $(h,\omega)$ is
identified with stage $(h+1,0)$. A configuration at stage $(L,0)$ is
\emph{accepting} if $v_j=t_j$ for all $j\in J_\sigma$.

\emph{Correctness.} We claim that an accepting configuration is reachable from
the initial configuration of $\sigma$ if and only if $G$ admits a
$\sigma$-routing. Given a $\sigma$-routing, move every token along its path;
since paths of different classes are disjoint, no move is blocked. Conversely,
in an accepting run the successive positions of token $j$ form a walk from
$s_{\sigma(j)}$ to $t_j$ along arcs, hence a path. If tokens $j,j'$ of different
classes both visit a vertex $v\in V_h$, then both occupy $v$ in the
configuration at stage $(h,0)$, which is valid: unmoved tokens at stage
$(h-1,r)$ lie in $V_{h-1}$, so the moves into $V_h$ are the only possible
conflicts, and these are excluded when they happen. Hence the paths form a
$\sigma$-routing. By \cref{thm:routing}, the instance is a yes-instance if and
only if some $\sigma$ with $\err_\calD(F_\sigma)\le\ell$ admits an accepting
run.

\emph{Deterministic running time.} There are at most $(\alpha+1)^\omega\le
|I|^\omega$ labelings, and $\err_\calD(F_\sigma)$ is computed in time
$O(|\calD|\omega)$. Since the partition is determined by $\sigma$ but the
configuration need not remember $\sigma$ after stage $(0,0)$, we merge
configurations from all labelings with admissible error: a configuration is a
tuple of at most $\omega$ positions together with a partition of the tokens
that are present, so there are at most $(n+1)^\omega B_\omega$ configurations
per stage, where $B_\omega\le\omega^\omega$ is the Bell number. There are
$L\omega\le n\omega$ stages, each configuration has at most $n$ successors, and
validity is checked in time $O(\omega)$. Storing the configurations of a stage
in a balanced search tree adds a factor polynomial in $\omega\log n$. The total
time is $(\alpha+1)^\omega|\calD|\omega+2^{O(\omega\log\omega)}n^{\omega+2}
(\log n)^{O(1)}=2^{O(\omega\log\omega)}|I|^{\omega+O(1)}$.

\emph{\XNLP.} A nondeterministic algorithm guesses $\sigma$
($\omega\lceil\log(\alpha+2)\rceil$ bits), computes $\err_\calD(F_\sigma)$ in one
scan of $\calD$ with a counter of $O(\log|I|)$ bits, and compares it with $\ell$
bit by bit. It then guesses the moves of the token game one at a time, storing
only the current configuration ($\omega$ positions, a class representative for
each token, and the stage) in $O(\omega\log|I|)$ bits, and verifies each arc by
scanning the input. The run takes $L\omega$ moves, each in polynomial time.
Both models are covered, since \cref{thm:routing} holds for both.
\end{proof}

\begin{proof}[Proof of \cref{cor:one-input}]
If $\alpha=1$, the inputs are $0$ and $1$, so every training set is
chain-ordered. For $\sigma:[\omega]\to\{1,\bot\}$, a $\sigma$-routing exists
if and only if every $t_j$ with $\sigma(j)=1$ is reachable from $s_1$ in $G$,
since paths from the same input may share vertices. Let $J$ be the set of
outputs reachable from $s_1$. Then $F_\sigma(0)=0$ and $F_\sigma(1)=y_\sigma$,
the indicator vector of $\sigma^{-1}(1)$, which ranges over all
$y\in\mathcal{Y}=\{y:y_j=0\text{ for }j\notin J\}$. By \cref{thm:routing} the
minimum error is
\[
\#\{(0,y)\in\calD:y\neq0\}+\#\{(1,y)\in\calD\}
-\max_{y\in\mathcal{Y}}\#\{(1,y)\in\calD\},
\]
with multiplicities. The maximum is attained at a label that occurs in $\calD$
and lies in $\mathcal{Y}$, or at $y=0\in\mathcal{Y}$. Sorting the labels of the
examples with input $1$ and computing $J$ by graph search take polynomial time.
\end{proof}

\section{The Hardness Reduction}
\label{app:reduction}

This appendix proves the two graph transformations of \cref{sec:chain} and
verifies that their composition is a parameterized reduction. We refer to the
index $i$ of a terminal pair as \emph{commodity}~$i$.

\subsection{Terminal Stubs}

Let $(H,(p_i,q_i)_{i\in[k]})$ be a \DAGEDP instance; terminals may coincide.
The stubbed instance $(H',(p_i^-,q_i^+)_{i\in[k]})$ adds, for each $i\in[k]$,
fresh vertices $p_i^-$ and $q_i^+$ and \emph{stub arcs} $p_i^-p_i$ and
$q_iq_i^+$.

\begin{lemma}
\label{lem:stubs}
$H'$ is acyclic; $p_1^-,\dots,p_k^-$ are pairwise distinct sources and
$q_1^+,\dots,q_k^+$ pairwise distinct sinks of $H'$; and $(H,(p_i,q_i))$ is a
yes-instance if and only if $(H',(p_i^-,q_i^+))$ is. Every $p_i^-\reach q_i^+$
path in $H'$ starts with the arc $p_i^-p_i$, ends with $q_iq_i^+$, and has at
least two arcs.
\end{lemma}

\begin{proof}
The new vertices have in-degree or out-degree zero, so they lie on no cycle.
The only arc leaving $p_i^-$ is $p_i^-p_i$ and the only arc entering $q_i^+$ is
$q_iq_i^+$; these are distinct, so every $p_i^-\reach q_i^+$ path starts and ends
with them. Removing them leaves a $p_i\reach q_i$ path in $H$ (of length zero if
$p_i=q_i$), since stub arcs of other commodities cannot be interior: $p_j^-$ is
a source and $q_j^+$ a sink. Conversely, a $p_i\reach q_i$ path extends by the
two stub arcs of commodity $i$. Stub arcs are private to their commodity, so
arc-disjointness is preserved in both directions.
\end{proof}

\subsection{The Augmented Line Graph}

From now on, $(H,(p_i,q_i))$ denotes the stubbed instance. Let $K$ be the
digraph of \cref{sec:chain}: a vertex $z_e$ for each arc $e$ of $H$, an arc
$z_e\to z_f$ whenever $\head(e)=\tail(f)$, fresh sources $s_i$ with arcs
$s_i\to z_e$ for all $e$ with $\tail(e)=p_i$, and fresh sinks $\hat t_i$ with
arcs $z_e\to\hat t_i$ for all $e$ with $\head(e)=q_i$.

\begin{lemma}
\label{lem:Kacyclic}
$K$ is acyclic, $s_1,\dots,s_k$ are pairwise distinct sources, and $\hat
t_1,\dots,\hat t_k$ are pairwise distinct sinks. Each $s_i$ has an out-arc and
each $\hat t_i$ an in-arc.
\end{lemma}

\begin{proof}
The fresh vertices lie on no cycle. A cycle $z_{e_1}\to\dots\to z_{e_m}\to
z_{e_1}$ would make $\tail(e_1),\dots,\tail(e_m),\tail(e_1)$ a closed walk of
positive length in $H$. The stub arcs give the arcs $s_i\to z_{p_i^-p_i}$ and
$z_{q_iq_i^+}\to\hat t_i$.
\end{proof}

\begin{lemma}[Arc capacity becomes vertex capacity]
\label{lem:linegraph}
$H$ has pairwise arc-disjoint paths $p_i\reach q_i$, $i\in[k]$, if and only if
$K$ has pairwise vertex-disjoint paths $s_i\reach\hat t_i$, $i\in[k]$.
\end{lemma}

\begin{proof}
($\Rightarrow$) A path $Q_i$ with arcs $e_1,\dots,e_r$, $r\ge2$, maps to the walk
$s_i,z_{e_1},\dots,z_{e_r},\hat t_i$, which is a path by
\cref{fact:walk,lem:Kacyclic}. Two images share a vertex $z_e$ only if both
paths use $e$, and the fresh terminals belong to one commodity.
($\Leftarrow$) Interior vertices of a path are neither sources nor sinks, and
$s_i\hat t_i$ is not an arc. Hence an $s_i\reach\hat t_i$ path has the form
$s_i,z_{e_1},\dots,z_{e_r},\hat t_i$ with $r\ge1$, $\tail(e_1)=p_i$,
$\head(e_r)=q_i$, and $\head(e_j)=\tail(e_{j+1})$. Then $e_1\cdots e_r$ is a
walk from $p_i$ to $q_i$, hence a path, and vertex-disjointness in $K$ gives
arc-disjointness in $H$.
\end{proof}

\subsection{Architecture Normalization}
\label{app:normalization}

\begin{lemma}[Architecture normalization]
\label{lem:normalization}
Let $K$ be a DAG with pairwise distinct designated sources $s_1,\dots,s_k$ and
pairwise distinct designated sinks $\hat t_1,\dots,\hat t_k$. One can compute
in polynomial time an architecture $G$ of depth $L\ge1$ with inputs
$s_0,\dots,s_k$ and outputs $t_0,\dots,t_k$ such that $K$ has pairwise
vertex-disjoint paths $s_i\reach\hat t_i$, $i\in[k]$, if and only if $G$ has
pairwise vertex-disjoint paths $s_i\reach t_i$, $i\in[k]$.
\end{lemma}

The construction has four steps.

\emph{Step 1 (dummy terminals).} Add fresh vertices $s_0,\hat t_0$, an arc $s_0u$
for every source $u$ of $K$ other than $s_1,\dots,s_k$, an arc $v\hat t_0$ for
every sink $v$ of $K$ other than $\hat t_1,\dots,\hat t_k$, and the arc $s_0\hat
t_0$. Call the result $K^+$.

\begin{claim}
\label{claim:step1}
$K^+$ is acyclic, its sources are exactly $s_0,\dots,s_k$, and its sinks are
exactly $\hat t_0,\dots,\hat t_k$. For $i,i'\in[k]$, every $s_i\reach\hat
t_{i'}$ path of $K^+$ is a path of $K$.
\end{claim}

\begin{proof}
$s_0$ has no in-arcs and $\hat t_0$ no out-arcs, so no cycle is created.
Undesignated sources and sinks of $K$ gain an in-arc or out-arc; $s_0$ is not a
sink and $\hat t_0$ is not a source, because of the arc $s_0\hat t_0$. An
$s_i\reach\hat t_{i'}$ path with $i,i'\ge1$ contains neither $s_0$, a source
different from its first vertex, nor $\hat t_0$, a sink different from its last
vertex; every new arc is incident to one of them.
\end{proof}

\emph{Step 2 (levels).} Let $\lambda(v)$ be the maximum number of arcs on a path
of $K^+$ ending at $v$. It is computable in linear time along a topological
order; $\lambda(v)=0$ exactly for sources; $\lambda(v)\ge\lambda(u)+1$ for every
arc $uv$; and $\lambda(v)\le|V(K^+)|-1=|V(K)|+1$.

\emph{Step 3 (subdivision).} Replace every arc $uv$ with
$g=\lambda(v)-\lambda(u)\ge2$ by a path $u\to x^{uv}_1\to\dots\to
x^{uv}_{g-1}\to v$ through fresh vertices with $\lambda(x^{uv}_j)=\lambda(u)+j$.
Call the result $K^{++}$. Every arc now joins consecutive levels, and since
subdivision vertices have in- and out-degree one, sources and sinks are
unchanged.

\begin{claim}
\label{claim:step3}
For $i\in[k]$, contracting subdivision paths is a bijection between
$s_i\reach\hat t_i$ paths of $K^{++}$ and of $K^+$, and it preserves and
reflects vertex-disjointness of families of such paths.
\end{claim}

\begin{proof}
A path of $K^+$ expands uniquely. A path of $K^{++}$ between vertices of $K^+$
that enters $x^{uv}_j$ must traverse the whole subdivision path of $uv$, whose
internal vertices have in- and out-degree one; contracting gives a path of
$K^+$. Two expanded paths sharing $x^{uv}_j$ both use $uv$ and share $u$; two
paths sharing a vertex of $K^+$ share it after contraction.
\end{proof}

\emph{Step 4 (output padding).} Let $L=\max_v\lambda(v)$; $L\ge1$ because of the
arc $s_0\hat t_0$. For $i\in\{0,\dots,k\}$, append to $\hat t_i$ a fresh path of
length $L-\lambda(\hat t_i)$ and let $t_i$ be its last vertex ($t_i=\hat t_i$ if
the length is zero). Call the result $G$, with layers
$V_h=\{v:\lambda(v)=h\}$, inputs $s_0,\dots,s_k$, and outputs
$t_0,\dots,t_k$.

\begin{claim}
\label{claim:valid}
$G$ is an architecture of depth $L\ge1$ with $k+1$ inputs and $k+1$ outputs.
\end{claim}

\begin{proof}
$G$ is acyclic, arcs join consecutive layers, the sources are exactly
$s_0,\dots,s_k$ (the vertices at level $0$), and the sinks are exactly
$t_0,\dots,t_k$, all at level $L$, because padding vertices other than $t_i$
have in- and out-degree one and $\hat t_i$ stops being a sink when its padding
path is nonempty.
\end{proof}

\begin{claim}
\label{claim:step4}
For $i\in[k]$, appending the padding path of $\hat t_i$ is a bijection between
$s_i\reach\hat t_i$ paths of $K^{++}$ and $s_i\reach t_i$ paths of $G$ that
preserves and reflects vertex-disjointness.
\end{claim}

\begin{proof}
Padding vertices of commodity $i$ can be entered only from $\hat t_i$, and $t_i$
is reachable only through them. Padding paths of different commodities are
disjoint.
\end{proof}

\begin{proof}[Proof of \cref{lem:normalization}]
By \cref{claim:valid}, $G$ is an architecture of depth $L\ge1$. By
\cref{claim:step1,claim:step3,claim:step4}, vertex-disjoint path families for
commodities $1,\dots,k$ correspond in $K$, $K^+$, $K^{++}$, and $G$. With
$n=|V(K)|$ and $m=|E(K)|$, Step~1 adds two vertices and at most $2n+1$ arcs,
Step~3 adds at most $L-1\le n$ vertices per arc of $K^+$, and Step~4 adds at
most $L\le n+1$ vertices per output. Hence $G$ has $O(n(m+n+k))$ vertices and
arcs and is computed in polynomial time.
\end{proof}

\subsection{Parameter, Size, and Both Models}
\label{app:parameter}

\begin{lemma}[Correctness]
\label{lem:correct}
Let $G$ be the architecture produced from a \DAGEDP instance. The instance is a
yes-instance if and only if some binary network on $G$ fits $\calD_k$ exactly,
and in that case a bias-free network does.
\end{lemma}

\begin{proof}
By \cref{lem:stubs,lem:linegraph,lem:normalization} (whose hypotheses hold by
\cref{lem:Kacyclic}), the instance is a yes-instance if and only if $G$ has
vertex-disjoint paths $s_i\reach t_i$, $i\in[k]$. Apply \cref{thm:oneflip}.
\end{proof}

\paragraph{Size and parameter.}
Let $n_H=|V(H)|+|E(H)|+k$. Stubbing adds $2k$ vertices and arcs; $K$ has
$|E(H)|+4k$ vertices and $O(n_H^2)$ arcs; by \cref{lem:normalization}, $G$ has
$O(n_H^3)$ vertices and arcs. The training set consists of $2(k+1)^2$ bits and
$\ell=0$. The produced instance has $\alpha=\omega=|\calD_k|=k+1$, so the
parameter maps to $\alpha+\omega=2k+2$, or to $\omega=k+1$. Hence the map is a
polynomial-time parameterized reduction from \DAGEDP to \twoQNNT, and by
\cref{lem:correct}, the same map is a reduction to \bfQNNT: yes-instances admit
bias-free solutions, and no-instances admit no solution even with biases.
Together with the \Wone-hardness of \DAGEDP \citep{slivkins2010edp}, this proves
\cref{thm:main}.

\paragraph{Unbounded structural quantities.}
In the produced instances, the depth, the width, the treewidth, and the number
of neurons and arcs of $G$ are not bounded by any function of $\alpha+\omega$.
Except for the depth, this is necessary unless $\FPT=\Wone$: \dQNNT is FPT
parameterized by $\tw(G)+\alpha+\omega$ \citep{ganian2026tractability}, and
width, number of neurons, and number of arcs each bound the treewidth of a
layered architecture from above up to a constant factor.

\paragraph{Witness structure.}
On every yes-instance, the bias-free solution of \cref{thm:oneflip} selects
exactly the arcs of $k$ vertex-disjoint paths, so every neuron has at most one
selected parent. Hardness therefore does not come from dense solutions. What is
unbounded is the active part of the network: even for $\alpha=\omega=1$,
fitting $\{(0,0),(1,1)\}$ on a directed path with $L$ arcs forces every weight
to one and every bias to zero, so all $L$ non-input neurons are active on
input~$1$ (\cref{prop:path}).

\section{Proofs for Bias-Free Training}
\label{app:biasfree}

\subsection{Source Sets}

\begin{lemma}[Source-set invariant]
\label{lem:invariant}
Let $\calP=(P_d)_{d\in R}$ be an $R$-path system and $H_\calP$ the union of its
arcs. For every $d\in R$ and every layer $h$, $S_h(d)$ is the set of inputs
$i$ such that $s_i$ reaches $v_h(d)$ in $H_\calP$.
\end{lemma}

\begin{proof}
For $h=0$, $v_0(i,j)=s_i$ is a source, reached only by itself. Let $h\ge0$ and
$v=v_{h+1}(d)$. Every arc of $H_\calP$ entering $v$ lies on some path $P_e$ through
$v$, and is then the arc $v_h(e)v$; conversely, each such arc lies in
$H_\calP$. The inputs reaching $v$ in $H_\calP$ are therefore the inputs reaching
some $v_h(e)$ with $v_{h+1}(e)=v$, which by induction is
$\bigcup\{S_h(e):v_{h+1}(e)=v\}=S_{h+1}(d)$. Several paths through the same arc
contribute the same set, so repetitions do not matter.
\end{proof}

With \cref{lem:invariant}, the proof of \cref{thm:reach} in the main text is
complete: there, ``the set of inputs reaching $t_j$ in $H_\calP$'' equals
$S_L(i,j)$ for every $(i,j)\in R$, since $v_L(i,j)=t_j$.

\subsection{The Algorithm}

\begin{proof}[Proof of \cref{thm:bf-alg}]
If $L=0$, evaluate the fixed network directly. Assume $L\ge1$.

\emph{Enumeration.} For each $R\subseteq[\alpha]\times[\omega]$, compute
$\err_\calD(F_R)$ in time $O(|\calD|\alpha\omega)$, and test whether $R$ is
realizable as follows. By \cref{obs:bf}, the instance is a yes-instance if and
only if some realizable $R$ has $\err_\calD(F_R)\le\ell$. The empty relation is
realized by $w\equiv0$.

\emph{Token game for $R\neq\emptyset$.} List $R$ as $d_1,\dots,d_p$, with
$d_q=(i_q,j_q)$. A configuration at stage $(h,r)$, $0\le h<L$, $0\le r\le p$,
assigns to every token $q\in[p]$ a vertex $v_q$ and a set $A_q\subseteq[\alpha]$,
where $v_q\in V_{h+1}$ for $q\le r$ and $v_q\in V_h$ for $q>r$. The initial
configuration at stage $(0,0)$ has $v_q=s_{i_q}$ and $A_q=\{i_q\}$. A move from
stage $(h,r-1)$ to $(h,r)$ chooses an arc $v_rv'$ with $v'\in V_{h+1}$ and sets
\[
B=A_r\cup\bigcup\{A_q:q<r,\ v_q=v'\},
\]
then $v_r\gets v'$, $A_r\gets B$, and $A_q\gets B$ for all $q<r$ with $v_q=v'$.
Stage $(h,p)$ is identified with stage $(h+1,0)$. A configuration at stage
$(L,0)$ is accepting if $v_q=t_{j_q}$ and $A_q=C_{j_q}$ for all $q$.

\emph{Invariant.} Fix paths $P_{d_1},\dots,P_{d_p}$ and move each token along
its path. By induction on $r$, at stage $(h,r)$ every unmoved token $q>r$ has
$A_q=S_h(d_q)$, and every moved token $q\le r$ has
$A_q=\bigcup\{S_h(d_{q'}):q'\le r,\ v_{h+1}(d_{q'})=v_{h+1}(d_q)\}$. For $r=0$
this is the outer induction hypothesis. In the move of token $r$, the tokens
$q<r$ at $v'$ all carry the same set, namely the union over the moved tokens at
$v'$, and $B$ adds $A_r=S_h(d_r)$; the other sets do not change. At $r=p$ every
token carries $S_{h+1}(d_q)$ by~\eqref{eq:propagate}.

\emph{Correctness.} If $R$ is realizable, \cref{thm:reach}(ii) provides paths
whose run ends in an accepting configuration, by the invariant. Conversely,
the positions of the tokens in an accepting run form an $R$-path system,
because tokens move along arcs from $s_{i_q}$ to $t_{j_q}$; by the invariant its
source sets satisfy $S_L(d_q)=C_{j_q}$, so $R$ is realizable by
\cref{thm:reach}. Selecting the arcs of the run and setting all other weights to
zero gives an explicit network.

\emph{Running time.} At each stage there are at most $n^p2^{\alpha p}$
configurations; each has at most $n$ successors, computed in time
$O(p(\alpha+\log n))$. There are $pL\le pn$ stages. Keeping the configurations
of a stage in a balanced search tree keyed by their $O(p(\alpha+\log n))$-bit
encoding, the time per relation is $2^{\alpha p}n^{p+2}\,(p\,\alpha\log
n)^{O(1)}$. Summing over $2^{\alpha\omega}$ relations with $p\le\alpha\omega$
and adding the error computations gives
\[
2^{\alpha\omega}\cdot O(|\calD|\alpha\omega)
+2^{\alpha^2\omega+\alpha\omega}\,n^{\alpha\omega+2}\,(\alpha\omega\log n)^{O(1)}
=2^{O(\alpha^2\omega)}\cdot|I|^{\alpha\omega+O(1)}.
\]
With $\kappa=\alpha+\omega$, we have $\alpha\omega\le\lfloor\kappa^2/4\rfloor$
and $\alpha^2\omega+\alpha\omega\le\kappa^3$, so the bound can also be written
as $2^{O(\kappa^3)}\cdot|I|^{\lfloor\kappa^2/4\rfloor+O(1)}$.
The algorithm assumes no bound on depth, width, degree, or treewidth.

\emph{\XNLP membership.} A nondeterministic algorithm guesses $R$ and stores
its $\alpha\omega$ bits. It computes $\err_\calD(F_R)$ in one scan of $\calD$,
using a counter of $O(\log|I|)$ bits, and compares it with $\ell$ bit by bit.
It then simulates the token game, guessing one move at a time and storing only
the current configuration and stage, which take
$p(\lceil\log n\rceil+\alpha)+O(\log n)=O(\alpha\omega(\alpha+\log|I|))$ bits;
arcs and layers are verified by scanning the input. The run has $pL$ moves,
each taking polynomial time. The space is at most $f(\alpha+\omega)\log|I|$ for
$f(\kappa)=O(\kappa^3)$, and the time is polynomial. By the correctness
argument, an accepting computation exists exactly on yes-instances.
\end{proof}

\section{Proofs for Trainable Biases}
\label{app:biases}

\subsection{The Fitting-Power Gap}
\label{app:gap}

\begin{proposition}[Biases overcome routing capacity]
\label{prop:gap}
There is an architecture with $\alpha=\omega=5$ and a training set of five
pairwise incomparable inputs that a binary network fits exactly, but no
bias-free binary network does.
\end{proposition}

\begin{proof}
Take $C_1,\dots,C_5$ to be the $2$-subsets of $[4]$ other than $\{3,4\}$. The
architecture has layers $\{s_1,\dots,s_5\}$, $\{z_1,\dots,z_4\}$,
$\{t_1,\dots,t_5\}$ and arcs $s_iz_r$ and $z_rt_i$ for $r\in C_i$. Every $z_r$
has an in-arc and an out-arc, because each $r\in[4]$ lies in some $C_i$, so the
architecture is valid. The training set is $\{(e_h,e_h):h\in[5]\}$, whose inputs
are pairwise incomparable.

With all weights one, $b(z_r)=0$, and $b(t_i)=1$: on $e_h$, the active hidden
vertices are $\{z_r:r\in C_h\}$, and $t_i$ fires if and only if $|C_i\cap
C_h|\ge2$, that is, $C_i=C_h$, that is, $i=h$. So the network fits exactly.

Suppose a bias-free network fits exactly, with reachability relation $R$. Then
$F_R(e_h)_j=\ind[h\in C'_j]$, where $C'_j$ is the $j$-th column of $R$, must
equal $\ind[h=j]$; so $R$ is the identity. Choose paths $P_i:s_i\reach t_i$ of
selected arcs. If $P_i$ and $P_j$ with $i\neq j$ shared a vertex, then $s_i$
would reach $t_j$. So the five paths are vertex-disjoint, but each passes
through a vertex of the middle layer, which has only four vertices.
\end{proof}

\subsection{Two Inputs}
\label{app:two}

\begin{proof}[Proof of \cref{thm:two-witness}]
Take a network realizing $(f_1,\dots,f_\omega)$ and bring it into the normal form
of \cref{lem:two-normal}; all functions are unchanged. For a vertex $v$, let
$\mathrm{cl}(v)$ be the set of vertices from which $v$ is reachable along
selected arcs, together with these arcs. Since copy vertices have one selected
parent of the same function and generation vertices have one $\fA$-parent and
one $\fB$-parent:
\begin{itemize}[nosep,leftmargin=*]
\item if $f_j\in\{\fA,\fB\}$, then $\mathrm{cl}(t_j)$ is a single path of copy
  vertices from $s_1$ (for $\fA$) or $s_2$ (for $\fB$) to $t_j$, since the only
  inputs computing $\fA$ and $\fB$ are $s_1$ and $s_2$;
\item if $f_j\in\{\fC,\fD\}$, following the unique selected parents back from
  $t_j$ passes through copy vertices of function $f_j$ until it reaches a
  generation vertex $g_j$, which must exist because inputs compute $\fA$ or
  $\fB$. Then $\mathrm{cl}(t_j)$ is the union of an $s_1\reach t_j$ path and an
  $s_2\reach t_j$ path that share the segment from $g_j$ to $t_j$;
\item if $f_j=0$, then $t_j$ has no selected parent.
\end{itemize}
Let $U=\bigcup_j\mathrm{cl}(t_j)$. By construction $U$ is closed under taking
all selected in-arcs of its vertices. Deselect all arcs outside $U$. Every vertex
of $U$ keeps all of its selected parents and its bias, so by induction over the
layers it computes the same function; outputs with $f_j=0$ still have no
selected parent. The selected arcs are the union of at most $2\omega$ paths, as
stated, and every non-input vertex on them is a copy or generation vertex.
\end{proof}

\subsection{The Two-Input Algorithm}
\label{app:two-alg}

\begin{proof}[Proof of \cref{thm:two-alg}]
For $\alpha=1$, every training set is chain-ordered, so \cref{cor:one-input}
gives polynomial time and \cref{thm:chain-alg} gives membership in \XNLP
parameterized by $\omega$. Let $\alpha=2$ and $L\ge1$.
A vector $f=(f_1,\dots,f_\omega)\in\{0,\fA,\fB,\fC,\fD\}^\omega$ determines the
predictions on every input, so $\err_\calD(f)$ is computed in time
$O(|\calD|\omega)$. By \cref{lem:monotone}, the instance is a yes-instance if and
only if some realizable $f$ has $\err_\calD(f)\le\ell$.

\emph{Tokens.} For an admissible $f$, create one token $(1,j)$ starting at $s_1$
for every $j$ with $f_j\in\{\fA,\fC,\fD\}$ and one token $(2,j)$ starting at
$s_2$ for every $j$ with $f_j\in\{\fB,\fC,\fD\}$; there are $p\le2\omega$
tokens, ordered arbitrarily. Every token has a position and a label. At the
start of a layer, the label of a token is the function of its vertex, in
$\{\fA,\fB,\fC,\fD\}$; initially $s_1$ has $\fA$ and $s_2$ has $\fB$.

\emph{Moves.} Within layer $h$, the tokens move one at a time. A token at $u$
with label $g$ moves along an arc $uv$, $v\in V_{h+1}$; its new label is
$\Phi=\{g\}\cup\bigcup\{\Phi_q:q\text{ already moved},\ v_q=v\}$, and all
already moved tokens at $v$ receive $\Phi$ too. After all tokens have moved,
every occupied $v\in V_{h+1}$ carries the set $\Phi_v$ of labels of the vertices
from which tokens entered $v$. Then every occupied $v$ chooses either bias $0$
and the function $\bigvee\Phi_v$ (the pointwise disjunction), or, only if
$\Phi_v=\{\fA,\fB\}$, bias $1$ and the function $\fC$; all tokens at $v$ receive
this function as their label. Accept if, after the last layer, every token
$(\cdot,j)$ is at $t_j$ with label $f_j$.

\emph{Soundness.} Given an accepting run, build a network: for an occupied
vertex $v$ with bias $0$, select every arc along which a token entered $v$; with
bias $1$, select one such arc from a vertex labeled $\fA$ and one from a vertex
labeled $\fB$; select nothing else, and set the biases as chosen (zero at
unoccupied vertices). By induction over the layers, every occupied vertex
computes its label. Its selected parents are occupied vertices of the previous
layer whose labels, by induction, are their functions. With bias $0$ it computes
the disjunction of their functions, which is $\bigvee\Phi_v$, regardless of how
many parents share a function; with bias $1$ it has exactly the two parents
$\fA$ and $\fB$ and computes $\fC$. Unoccupied non-input vertices have no selected parent
and compute $0$. Every $t_j$ with $f_j\neq0$ is occupied and computes $f_j$.
Every $t_j$ with $f_j=0$ is unoccupied, because tokens end only at their own
outputs and outputs are not interior to paths; it computes $0$.

\emph{Completeness.} If $f$ is realizable, let $N$ be the network of
\cref{thm:two-witness} and move every token along the corresponding path of
$N$: token $(1,j)$ along the $s_1\reach t_j$ path and $(2,j)$ along the
$s_2\reach t_j$ path. Every selected arc of $N$ lies on a token path, so the
vertices entering an occupied $v$ in the run are exactly the selected parents of
$v$ in $N$. By induction, the labels are the functions in $N$, and $v$ chooses
the bias it has in $N$: a copy vertex has $\Phi_v=\{g\}$ and bias $0$; a
generation vertex has $\Phi_v=\{\fA,\fB\}$ and bias $0$ for $\fD$ or $1$ for
$\fC$. Hence the run is accepting.

\emph{Running time and space.} A configuration consists of $p\le2\omega$
positions, at most $n^{2\omega}$ choices, and $p$ labels, each a function or a
subset of $\{\fA,\fB,\fC,\fD\}$, at most $2^{O(\omega)}$ choices. Each stage
has at most $n$ successors per configuration, the end-of-layer bias choices
give at most $2^p$ successors, and there are $O(pL)$ stages. With the
$5^\omega$ vectors and the error computations, the time is
$5^\omega|\calD|\omega+2^{O(\omega)}n^{2\omega+2}(\log n)^{O(1)}=
2^{O(\omega)}|I|^{2\omega+O(1)}$. The nondeterministic version guesses $f$
($3\omega$ bits), checks its error with an $O(\log|I|)$-bit counter, and guesses
moves and bias choices while storing one configuration in $O(\omega\log|I|)$
bits; its running time is polynomial. Hence the problem is in \XNLP
parameterized by $\omega$.
\end{proof}

\subsection{Conjunctions}
\label{app:conj}

\begin{proof}[Proof of \cref{thm:conj}]
Write $t=t_1$ and $f_S=\bigwedge_{i\in S}x_i$.

\emph{Sufficiency.} Let $T$ be an in-arborescence as in the statement. Select
exactly its arcs, give in-degree-one vertices of $T$ bias $0$ and in-degree-two
vertices bias $1$, and give all other vertices bias $0$ (they have no selected
parent). The subtrees of the two children of a vertex of $T$ are disjoint, so by
induction every vertex of $T$ computes the conjunction of the inputs at the
leaves of its subtree, and $t$ computes $f_S$.

\emph{Necessity.} Fix a binary network computing $f_S$. Let $z=\ind_S$ and
$z^{(i)}=\ind_{S\setminus\{i\}}$ for $i\in S$. Call $v$ \emph{active} if
$a_v(z)=1$, and for active $v$ let $Q_v=\{i\in S:a_v(z^{(i)})=0\}$. Then $t$ is
active with $Q_t=S$; an input $s_j$ is active exactly when $j\in S$, and then
$Q_{s_j}=\{j\}$. For an active non-input vertex $v$:
\begin{itemize}[nosep,leftmargin=*]
\item if $b(v)=0$, then $v$ has an active selected parent, and $Q_v\subseteq Q_u$
  for every active selected parent $u$: if $a_v(z^{(i)})=0$, all selected
  parents are inactive on $z^{(i)}$;
\item if $b(v)=1$, then $v$ has two distinct active selected parents $u_1,u_2$,
  and $Q_v\subseteq Q_{u_1}\cup Q_{u_2}$: on $z^{(i)}$ with $i\in Q_v$, at most
  one selected parent is active.
\end{itemize}
Route demands backwards. Start with demand $i$ at $t$ for every $i\in S$, and
process the vertices in reverse topological order, maintaining that every
demand at a vertex $v$ is an element of $Q_v$ and $v$ is active. At a non-input
vertex with $b(v)=0$, send all its demands to one active selected parent $u$;
with $b(v)=1$, send each demand $i$ to one of $u_1,u_2$ whose set contains $i$.
Each demand $i$ traces a path backwards to an active input $s_j$ with $i\in
Q_{s_j}=\{j\}$, so it ends at $s_i$. Let $U$ be the union of the arcs used. Every
non-input vertex has in-degree at most two in $U$, and every vertex of $U$
reaches $t$ in $U$.

Keep, for every vertex of $U$ other than $t$, one out-arc in $U$. Following kept
arcs from any vertex of $U$ leads to the last layer, whose only vertex in $U$ is
$t$; so the kept arcs form an in-arborescence rooted at $t$, and every $s_i$,
$i\in S$, still reaches $t$. Restrict it to the vertices reachable from
$\{s_i:i\in S\}$. Every non-root vertex keeps its out-arc, every non-input
vertex keeps an in-arc, in-degrees remain at most two, and the only inputs in
$U$ are active, hence in $\{s_i:i\in S\}$. The result is the required
in-arborescence.

\emph{Algorithm.} Let $r=|S|$. Maintain a partition of $S$ into blocks placed on
distinct vertices of the current layer, starting with $\{i\}$ on $s_i$. In each
layer every block moves along an arc; at most two blocks may enter a vertex,
and two blocks entering the same vertex merge. Accept if the single block $S$
reaches $t$. Accepting runs correspond exactly to layered in-arborescences with
leaves $\{s_i:i\in S\}$ and in-degree at most two. There are at most
$B_rn^{r}$ configurations per layer and at most $n^r$ transitions from each, so
the time is $B_r^2n^{2r+O(1)}=|I|^{O(r)}$.
\end{proof}

\section{Supplementary Results}
\label{app:supp}

The results in this appendix are not used by the main theorems. They delimit
the proof techniques and record examples referred to in the main text.

\subsection{Local Function Sets Do Not Compose}

\begin{proposition}[Independently attainable functions need not coexist]
\label{prop:obstruction}
There is an architecture with $\alpha=3$ and $\omega=1$ in which two vertices
$v_1,v_2$ can be made to compute $x_1\land x_2$ and $x_1\land x_3$,
respectively, each by some binary network, and the output can compute the
disjunction of any two functions computed at $v_1$ and $v_2$, yet no binary
network computes $x_1\land(x_2\lor x_3)$ at the output.
\end{proposition}

\begin{proof}
Let the layers be $\{s_1,s_2,s_3\}$, $\{u\}$, $\{v_1,v_2\}$, $\{t\}$ with arcs
$s_1u,s_2u,s_3u,uv_1,uv_2,v_1t,v_2t$. With $w(s_1u)=w(s_2u)=1$, $w(s_3u)=0$,
$b(u)=1$, and $v_1$ copying $u$, the vertex $v_1$ computes $x_1\land x_2$;
symmetrically, $v_2$ can compute $x_1\land x_3$; and $t$ with bias $0$ computes
the disjunction of $v_1$ and $v_2$. In any single network, however, $v_1$ and
$v_2$ have the unique parent $u$ and compute $0$ or $g_u$, the function of $u$,
so $t$ computes $0$ or $g_u$. No single neuron $u$ computes
$x_1\land(x_2\lor x_3)$: with bias $0$, the zero values at $e_1,e_2,e_3$ force
all weights to zero; with bias $1$, the values one at $(1,1,0)$ and $(1,0,1)$
force all three weights to one, and then $u$ fires on $(0,1,1)$.
\end{proof}

An algorithm that propagates, for every vertex, the set of functions it can
compute in isolation, and combines these sets independently at each vertex,
would wrongly accept this instance. The algorithms of
\cref{thm:bf-alg,thm:two-alg} avoid this by fixing a single function per
occupied vertex.

\subsection{Bounded Support of Aggregation Trees}

Every non-input neuron can copy any one of its in-neighbors (weight one on that
arc, zero elsewhere, bias zero). Consequently, if a profile already occurs at an
ancestor of an output, it can be copied to the output, and a trainable
aggregation gadget can ``select'' one convenient input instead of checking all
of them. The following proposition makes one version of this limitation
precise. It concerns reprogramming part of a fixed network and is a statement
about proof strategies for \paraNP-hardness, not a complexity bound.

Fix a set $X$ of $m$ input vectors and a binary network on $G$. Let $Q$ be a
set of non-input vertices with a root $r\in Q$ such that every $q\in
Q\setminus\{r\}$ has exactly one out-arc in $G$, and its head lies in $Q$. Then
$Q$ induces an in-arborescence rooted at $r$, and no vertex outside $Q$ that
has an arc into $Q$ is a descendant of a vertex of $Q$ (otherwise $G$ would
contain a cycle through $r$). Hence the profiles on $X$ of all vertices outside
$Q$ that feed $Q$ do not depend on the parameters of $Q$, that is, on the
weights of arcs entering $Q$ and the biases of vertices of $Q$. Arcs entering
$Q$ from outside are \emph{ports}.

\begin{proposition}[Bounded support]
\label{prop:support}
In this setting, for every profile $\pi\in\{0,1\}^X$ that $r$ attains for some
choice of the parameters of $Q$, there is such a choice in which $r$ has
profile $\pi$ and at most $(2m)^{2^m}$ selected ports have a selected path to
$r$. If every vertex of $Q$ has in-degree at most two in $G$, the bound is
$2^{2^m}$.
\end{proposition}

\begin{proof}
If $\pi=0$, deselect all arcs into $r$. Otherwise, among all parameter choices
for $Q$ giving $r$ the profile $\pi$, take one with the fewest selected arcs
entering $Q$. Every selected arc entering $Q$ then has a selected path to $r$,
and its tail has a nonzero profile; otherwise deselecting it would not change
the profile of $r$. Hence every vertex $q$ of the \emph{cone} (the vertices with
a selected path to $r$) has a nonzero profile. A vertex $q\in Q$ of the cone with
bias $b$ has, for each $x\in X$ with $a_q(x)=1$, at least $1+b$ active selected
parents; keeping $1+b$ of them for every such $x$ preserves the profile of $q$
and keeps at most $2m$ arcs, so by minimality $q$ has at most $2m$ selected
parents. A vertex with exactly one selected parent has bias $0$ and copies it.

Unfold the cone into a tree whose leaves are the selected ports and whose
internal nodes are the cone vertices of $Q$, and suppress internal nodes with
one child. Every remaining internal node (a \emph{branching} vertex) has between
$2$ and $2m$ children. Suppose a branching vertex $u$ and a proper descendant
$v$ in this tree (a branching vertex or a port tail) have the same profile. Let
$v=y_0,y_1,\dots,y_s=u$ be the selected path from $v$ to $u$, and reprogram each
$y_1,\dots,y_s$ to copy its predecessor on the path. Then $u$ has its old
profile. The profiles of $y_1,\dots,y_{s-1}$ may change, but each of them has
its only out-arc on the path, so no other vertex is affected. Every $y_l$ keeps
one selected parent instead of at least one, and $u$ loses at least one, which
contradicts minimality. Hence the branching vertices and the leaf on every
root-to-leaf path of the tree have pairwise distinct nonzero profiles, so there
are fewer than $2^m$ branching vertices on any such path, and the tree has at
most $(2m)^{2^m}$ leaves; with in-degree at most two, at most $2^{2^m}$.
\end{proof}

Consequently, suppose a reduction attaches to an output an aggregation tree
fed by $q>(2m)^{2^m}$ independent check signals, meaning that each check can be
switched from a good profile to a bad one while all other profiles feeding the
tree stay fixed. Then the tree cannot force all checks to be good: by
\cref{prop:support}, some realization of an accepting profile ignores one of the
checks, and switching that check alone leaves the output unchanged. This does
not cover correlated checks, aggregation networks with shared vertices, or
vertices of the gadget that feed other parts of the network.

\subsection{Signed Weights}

The routing view relies on nonnegative weights.

\begin{proposition}[Signed weights break onset inheritance]
\label{prop:signed}
For $d=3$, there is a zero-bias network with weights in $\Z_3=\{-1,0,1\}$ and
activation $\relu_3(z)=\min\{\max\{z,0\},1\}$ such that, on the chain
$(0,0)<(1,0)<(1,1)$, an output switches on at the second step although no
in-neighbor switches on there.
\end{proposition}

\begin{proof}
Take $c=\relu_3(x_1)$, $v=\relu_3(x_1-x_2)$, and $t=\relu_3(c-v)$, with layers
$\{s_1,s_2\}$, $\{c,v\}$, $\{t\}$. On the chain, $c=(0,1,1)$, $v=(0,1,0)$, and
$t=(0,0,1)$: at the second step $t$ switches on, $c$ stays constant, and $v$
switches off.
\end{proof}

This concerns only the proof technique; it does not show that \dQNNT is easier
for $d\ge3$.

\subsection{Unboundedly Many Active Neurons}

\begin{proposition}
\label{prop:path}
Let $G$ be the directed path $v_0\to v_1\to\dots\to v_L$ with input $v_0$ and
output $v_L$. The unique binary network that fits $\{(0,0),(1,1)\}$ exactly has
all weights one and all biases zero.
\end{proposition}

\begin{proof}
On input $1$, by downward induction on $j$, $a_{v_j}=1$ requires
$w(v_{j-1}v_j)a_{v_{j-1}}\ge1+b(v_j)$, which forces $a_{v_{j-1}}=1$,
$w(v_{j-1}v_j)=1$, and $b(v_j)=0$. This network fits both examples.
\end{proof}

\subsection{Computational Verification}

The ancillary file \texttt{anc/sanity\_checks.py} compares, by exhaustive
enumeration of all binary parameter assignments on small random architectures,
the following statements with brute force: \cref{lem:disjunction} and the
token algorithm of \cref{thm:bf-alg} (all relations $R$); \cref{thm:chain} with
its profile-preserving transformation and \cref{thm:routing} on random chains;
the token algorithm of \cref{thm:two-alg} (all output-function vectors); the
characterization and algorithm of \cref{thm:conj};
\cref{prop:gap,prop:obstruction,ex:bottleneck}; and the complete path system
behind \cref{fig:witness}b. On about $950$ random
instances, all checks agreed. These experiments are a consistency check on the
statements, not part of their proofs.

\end{document}